%% file: neurips_2021.tex
\documentclass{article}

\PassOptionsToPackage{numbers, compress}{natbib}

\usepackage[final]{neurips_2021}

\usepackage[utf8]{inputenc} 
\usepackage[T1]{fontenc}    
\usepackage{hyperref}       
\usepackage{url}            
\usepackage{booktabs}       
\usepackage{amsfonts}       
\usepackage{nicefrac}       
\usepackage{microtype}      
\usepackage{xcolor}         
\usepackage{graphicx}
\usepackage{amsmath}
\usepackage{multirow}
\usepackage{bbm}
\usepackage{subfigure}
\usepackage{mathtools}
\usepackage{capt-of}
\usepackage{amsmath,amsthm,amssymb}
\usepackage{bm}

\usepackage[ruled]{algorithm2e}
\newtheorem{mydef}{Definition}

\newtheorem{assumption}{Assumption}[section]
\newtheorem{lemma}{Lemma}[section]
\newtheorem{theorem}{Theorem}[section]
\newtheorem{innercustomthm}{Theorem}
\newenvironment{customthm}[1]
  {\renewcommand\theinnercustomthm{#1}\innercustomthm}
  {\endinnercustomthm}

\DeclareMathOperator*{\argmin}{argmin}

\title{Overcoming Catastrophic Forgetting in Incremental Few-Shot Learning by Finding Flat Minima}

\author{%
  Guangyuan Shi,$^{*}$ Jiaxin Chen\thanks{Equal contribution}, ~Wenlong Zhang, Li-Ming Zhan, Xiao-Ming Wu\thanks{Corresponding author}
  \\
  Department of Computing\\
  The Hong Kong Polytechnic University\\
  \texttt{\{guang-yuan.shi, jiax.chen, wenlong.zhang, lmzhan.zhan\}@connect.polyu.hk}\\
  \texttt{xiao-ming.wu@polyu.edu.hk}\\
}

\begin{document}

\maketitle

\input{Abstract}
\input{Introduction}

\input{Related_Work}

\input{Methodology}

\input{Convergence_Analysis}

\input{Experiment}

\input{Conclusion}

\newpage
\begin{ack}
We would like to thank the anonymous reviewers for their insightful and helpful comments. This research was supported by the grant of DaSAIL project P0030935 funded by PolyU/UGC.
\end{ack}
\bibliographystyle{plain}

\section*{Checklist}

\begin{enumerate}

\item For all authors...
\begin{enumerate}
  \item Do the main claims made in the abstract and introduction accurately reflect the paper's contributions and scope?
    \answerYes{}
  \item Did you describe the limitations of your work?
    \answerYes{See Section~\ref{sec:conclusion}.}
  \item Did you discuss any potential negative societal impacts of your work?
    \answerNA{}
  \item Have you read the ethics review guidelines and ensured that your paper conforms to them?
    \answerYes{}
\end{enumerate}

\item If you are including theoretical results...
\begin{enumerate}
  \item Did you state the full set of assumptions of all theoretical results?
    \answerYes{See Section~\ref{sec:theory}.}
	\item Did you include complete proofs of all theoretical results?
    \answerYes{See Appendix~A.1.}
\end{enumerate}

\item If you ran experiments...
\begin{enumerate}
  \item Did you include the code, data, and instructions needed to reproduce the main experimental results (either in the supplemental material or as a URL)?
    \answerYes{The link to the source code is provided in the abstract.}
  \item Did you specify all the training details (e.g., data splits, hyperparameters, how they were chosen)?
    \answerYes{See Section~\ref{sec:exp_setup} and Appendix~A.2.}
	\item Did you report error bars (e.g., with respect to the random seed after running experiments multiple times)?
    \answerYes{See Appendix~A.3.}
	\item Did you include the total amount of compute and the type of resources used (e.g., type of GPUs, internal cluster, or cloud provider)?
    \answerYes{See Section~\ref{sec:exp_setup}}
\end{enumerate}

\item If you are using existing assets (e.g., code, data, models) or curating/releasing new assets...
\begin{enumerate}
  \item If your work uses existing assets, did you cite the creators?
    \answerYes{We have cited~\cite{krizhevsky2009learning,vinyals2016matching,welinder2010caltech}.}
  \item Did you mention the license of the assets?
    \answerNA{}
  \item Did you include any new assets either in the supplemental material or as a URL?
    \answerNA{}
  \item Did you discuss whether and how consent was obtained from people whose data you're using/curating?
    \answerNA{}
  \item Did you discuss whether the data you are using/curating contains personally identifiable information or offensive content?
    \answerNA{}
\end{enumerate}

\item If you used crowdsourcing or conducted research with human subjects...
\begin{enumerate}
  \item Did you include the full text of instructions given to participants and screenshots, if applicable?
    \answerNA{}
  \item Did you describe any potential participant risks, with links to Institutional Review Board (IRB) approvals, if applicable?
    \answerNA{}
  \item Did you include the estimated hourly wage paid to participants and the total amount spent on participant compensation?
    \answerNA{}
\end{enumerate}

\end{enumerate}

\newpage

\appendix
\input{Appendix}

\end{document}

%% file: Abstract.tex
\begin{abstract}
 This paper considers incremental few-shot learning, which requires a model to continually recognize new categories with only a few examples provided. Our study shows that existing methods severely suffer from catastrophic forgetting, a well-known problem in incremental learning, which is aggravated due to data scarcity and imbalance in the few-shot setting. Our analysis further suggests that to prevent catastrophic forgetting, actions need to be taken in the primitive stage -- the training of base classes instead of later few-shot learning sessions. Therefore, we propose to search for flat local minima of the base training objective function and then fine-tune the model parameters within the flat region on new tasks. In this way, the model can efficiently learn new classes while preserving the old ones. Comprehensive experimental results demonstrate that our approach outperforms all prior state-of-the-art methods and is very close to the approximate upper bound. The source code is available at \url{https://github.com/moukamisama/F2M}.
\end{abstract}

%% file: Introduction.tex
\section{Introduction}\label{sec:introduction}

\textbf{Why study incremental few-shot learning?} Incremental learning enables a model to continually learn new concepts from new data without forgetting previously learned knowledge. Rooted from real-world applications, this topic has attracted a significant amount of interest in recent years~\citep{chaudhry2018riemannian,kuzborskij2013n,Lwf,icarl,kemker2017fearnet}. Incremental learning assumes sufficient training data is provided for new classes, which is impractical in many application scenarios, especially when the new classes are rare categories which are costly or difficult to collect. This motivates the study of incremental few-shot learning, a more difficult paradigm that aims to continually learn new tasks with only a few examples. 

\textbf{Challenges.} The major challenge for incremental learning is \emph{catastrophic} \emph{forgetting}~\citep{goodfellow2013empirical,kirkpatrick2017overcoming,mccloskey1989catastrophic}, which refers to the drastic performance drop on previous tasks after learning new tasks. This phenomenon is caused by the inaccessibility to previous data while learning on new data. Catastrophic forgetting presents a bigger challenge for incremental few-shot learning. Due to the small amount of training data in new tasks, the model tends to severely overfit on new classes while quickly forgetting old classes, resulting in catastrophic performance.
 
\textbf{Current research.} The study of incremental few-shot learning has just started
 ~\citep{TOPIC, ren2019incremental, zhao2006mgsvf,cheraghian2021semantic,IDLVQC, FSLL,zhang2021few}. Current works mainly borrow ideas from research in incremental learning to overcome the  forgetting problem, by enforcing strong constraints on model parameters to penalize the changes of parameters~\citep{FSLL,kirkpatrick2017overcoming,zenke2017continual}, or by saving a small amount of exemplars from old classes and adding constraints on the exemplars to avoid forgetting~\citep{icarl,NCM,castro2018end}. However, in our empirical study, we find that an intransigent model that only trains on base classes and does not tune on new tasks consistently outperforms state-of-the-art methods, including a joint-training method~\citep{TOPIC} that uses all encountered data for training and hence suffers from severe data imbalance. This observation motivates us to address this harsh problem from a different angle.

\textbf{Our solution.} Unlike existing solutions that try to overcome the catastrophic forgetting problem during the process of learning new tasks, we adopt a different approach by considering this issue during the training of base classes. Specifically, we propose to search for flat local minima of the base training objective function. For any parameter vector in the flat region around the minima, the loss is small, and the base classes are supposed to be well separated. The flat local minima can be found by adding random noise to the model parameters for multiple times and jointly optimizing multiple loss functions. During the following incremental few-shot learning stage, we fine-tune the model parameters within the flat region, which can be achieved by clamping the parameters after updating them on few-shot tasks. In this way, the model can efficiently learn new classes while preserving the old ones. Our key contributions are summarized as follows:

\begin{itemize}

\item We conduct a comprehensive empirical study on existing incremental few-shot learning methods and discover that a simple baseline model that only trains on base classes outperforms state-of-the-art methods, which demonstrates the severity of catastrophic forgetting.

\item We propose a novel approach for incremental few-shot learning by addressing the catastrophic forgetting problem in the primitive stage. Through finding the flat minima region during training on base classes and fine-tuning within the region while learning on new tasks, our model can overcome catastrophic forgetting and avoid overfitting.

\item Comprehensive experimental results on CIFAR-100, \emph{mini}ImageNet, and CUB-200-2011 show that our approach outperforms all state-of-the-art methods and achieves performance that is very close to the approximate upper bound.
\end{itemize}

%% file: Related_Work.tex
\section{Related Work}

\textbf{Few-shot learning} aims to learn to generalize to new categories with a few labeled samples in each class. Current few-shot methods mainly include optimization-based methods~\cite{finn2017model,jamal2019task,liu2020ensemble,ravi2016optimization,sun2020meta,sun2019meta,yue2020interventional} and metric-based methods~\cite{gidaris2018dynamic,hou2019cross,PN,vinyals2016matching,ye2021learning,zhang2020deepemd,zhang2020deepemdd,ye2020few}. Optimization-based methods can achieve fast adaptation to new tasks with limited samples by learning a specific optimization algorithm.
Metric-based approaches exploit different distance metrics such as L2 distance~\cite{PN}, cosine similarity~\cite{vinyals2016matching}, and DeepEMD~\cite{zhang2020deepemd} in the learned metric/embedding space to measure the similarity between samples. 
Recently, Tian~\emph{et al.}~\cite{Tian2020RethinkingFI} find that standard supervised training 
can learn a good metric space 
for unseen classes,
which echoes with our observation on the proposed baseline model in Sec.~\ref{sec:analysis}.

\textbf{Incremental learning} focuses on the challenging problem of continually learning to recognize new classes in new coming data without forgetting old classes \cite{chaudhry2018efficient,chen2020mitigating,Dhar_2019_CVPR, liu2021Adaptive}. Previous research mainly includes multi-class incremental learning \cite{castro2018end, rajasegaran2020itaml, hu2021distilling,liu2020mnemonics,Yu2020SemanticDC, liu2021Adaptive} and multi-task incremental learning \cite{hu2018overcoming,Lwf,riemer2018learning}. To overcome the catastrophic forgetting problem, some attempts propose to impose strong constraints on model parameters by penalizing the changes of parameters~\cite{kirkpatrick2017overcoming,aljundi2018memory}.
Other attempts try to enforce constraints on the exemplars of old classes by restricting the output logits~\cite{icarl} or penalizing the changes of embedding angles~\cite{NCM}. In this work, our empirical study shows that imposing strong constraints on the arriving new classes may not be a promising way to tackle incremental few-shot learning, due to the scarcity of training data for new classes.

\textbf{Incremental few-shot learning}~\cite{TOPIC, ren2019incremental, zhao2006mgsvf,cheraghian2021semantic,IDLVQC} aims to incrementally learn from very few samples. 
TOPCI~\cite{TOPIC} proposes a neural gas network to learn and preserve the topology of the feature manifold formed by different classes. 
FSLL~\cite{FSLL} only selects few model parameters for incremental learning and ensures the parameters are close to the optimal ones.
To overcome catastrophic forgetting, IDLVQC~\cite{IDLVQC} imposes constraints on the saved exemplars of each class by restricting the embedding drift, and
Zhang~\emph{et al.}~\cite{zhang2021few} propose to fix the embedding network for incremental learning. Similar to the finding of Zhang~\emph{et al.}, we also discover that an intransigent model 
that simply does not adapt to new tasks can outperform prior state-of-the-art methods.

\textbf{Robust optimization.}
It has been found that flat local minima leads to better generalization capabilities than sharp minima in the sense that a flat minimizer is more robust when the test loss is shifted due to random perturbations~\cite{DBLP:conf/nips/HochreiterS94,DBLP:conf/colt/HintonC93,Jiang*2020Fantastic}.
A substantial body of methods~\cite{DBLP:journals/pnas/BaldassiPZ20,pittorino2021entropic,DBLP:conf/uai/DziugaiteR17,DBLP:conf/nips/HeHY19} have been proposed to optimize neural networks towards flat local minima.
In this paper, we show that for incremental few-shot learning, finding flat minima in the base session and tuning the model within the flat region on new tasks can significantly mitigate catastrophic forgetting.

%% file: Methodology.tex
\section{Severity of Catastrophic Forgetting in Incremental Few-Shot Learning}\label{sec:analysis}

\subsection{Problem Statement}
Incremental few-shot learning (IFL) aims to continually learn to recognize new classes with only few examples. Similar to incremental learning (IL), an IFL model is trained by a sequence of training sessions $\{\mathcal{D}^1, \cdots, \mathcal{D}^t\}$, where  $\mathcal{D}^t=\{z_i=(x_i^t, y_i^t)\}_i$ is the training data of session $t$ and  $x_i^t$ is an example of class $y_i^t\in\mathcal{C}^{t}$ (the class set of session $t$). In IFL, the base session $\mathcal{D}^1$ usually contains a large number of classes with sufficient training data for each class, while the following sessions ($t\geq 2$) only have a small number of classes with few training samples per class, e.g., $\mathcal{D}^t$ is often presented as an $N$-way $K$-shot task with small $N$ and $K$. The key difference between IL and IFL is, for IL, sufficient training data is provided in each session. Similar to IL, in each training session $t$ of IFL, the model has only access to the training data $\mathcal{D}^t$ and possibly a small amount of saved exemplars from previous sessions. When the training of session $t$ is completed, the model is evaluated on test samples from all encountered classes $\mathcal{C}=\bigcup_{i=1}^t \mathcal{C}^{i}$, where it is assumed that there is no overlap between the classes of different sessions, i.e., $\forall i,j$ and $i\neq j$, $\mathcal{C}^i \bigcap \mathcal{C}^j = \emptyset$. 

\textbf{Catastrophic forgetting.} IFL is undoubtedly a more challenging problem than IL due to the data scarcity setting. IL suffers from catastrophic forgetting, a well-known phenomenon and long-standing issue, which refers to the drastic drop in test performance on previous (old) classes, caused by the inaccessibility of old data in the current training session. Unfortunately, catastrophic forgetting is an even bigger issue for IFL, because data scarcity makes it difficult to adapt well to new tasks and learn new concepts, while the adaptation process could easily lead to the forgetting of base classes. In the following, we illustrate this point by evaluating a simple baseline model for IFL.

\subsection{A Simple Baseline Model for IFL}

We consider an intransigent model that simply does not adapt to new tasks. Particularly, the model only needs to be trained in the base session $\mathcal{D}^1$ and is directly used for inference in all sessions. 

\textbf{Training ($t=1$).} We train a feature extractor $f$ parameterized by $\phi$ with a fully-connected layer as classifier by minimizing the standard cross-entropy loss using the training examples of $\mathcal{D}^1$. The feature extractor $f$ is \textit{fixed} for the following sessions ($t\geq 2$) without any fine-tuning on new classes. 

\textbf{Inference (test).}
In each session, the inference is conducted by a simple nearest class mean (NCM) classification algorithm~\citep{mensink2013distance}. Specifically, all the training and test samples are mapped to the embedding space of the feature extractor $f$, and Euclidean distance $d(\cdot,\cdot)$ is used to measure the similarity between them. The classifier is given by
\begin{align}
    c_k^\star = \argmin_{c\in\mathcal{C}} d(f(x;\phi), p_c),\ \text{where} \
    p_c = \frac{1}{N_c} \sum_{i}\mathbbm{1}{(y_i=c)}f(x_i;\phi)\label{eq:prototypes},
\end{align}
where $\mathcal{C}$ denotes all the encountered classes, $p_c$ refers to the prototype of class $c$ (the mean vector of all the training samples of class $c$ in the embedding space), and $N_c$ denotes the number of the training images of class $c$. Note that we \textit{save} the prototypes of all classes in $\mathcal{C}^{t}$ for later evaluation. 

\textbf{The baseline model outperforms state-of-the-art IFL and IL methods.} We compare the above baseline model against state-of-the-art IFL methods including FSLL~\citep{FSLL}, IDLVQC~\citep{IDLVQC} and TOPIC~\citep{TOPIC}, IL methods including Rebalance~\citep{NCM} and iCarl~\citep{icarl}, and a joint-training method that uses all previously seen data including the base and the following few-shot tasks for training, for IFL. The performance is evaluated on miniImageNet, CIFAR-100, and CUB-200. We tune the methods re-implemented by us to the best performance. For the other methods, we use the results reported in the original papers. The experimental details are provided in Sec.~\ref{sec:exp}.
As shown in Fig.~\ref{fig_table1}, the baseline model consistently outperforms all the compared methods including the joint-training method (which suffers from severe data imbalance) on every dataset\footnote{We notice that a similar observation is made in a newly released paper~\citep{zhang2021few}.}. The fact that an intransigent model performs best suggests that 
\begin{itemize}
    \item For IFL, preserving the old (base classes) may be more critical than adapting to the new. Due to data scarcity, the performance gain on new classes is limited and cannot make up for the significant performance drop on base classes. 
    \item Prior works~\citep{TOPIC,IDLVQC,FSLL,NCM,icarl} that enforce strong constraints on model parameters or exemplars during fine-tuning on new classes cannot effectively prevent catastrophic forgetting in IFL, indicating that actions may need to be taken in the base training stage.
\end{itemize}

\begin{figure}
    \centering
  \includegraphics[width=1.0\linewidth]{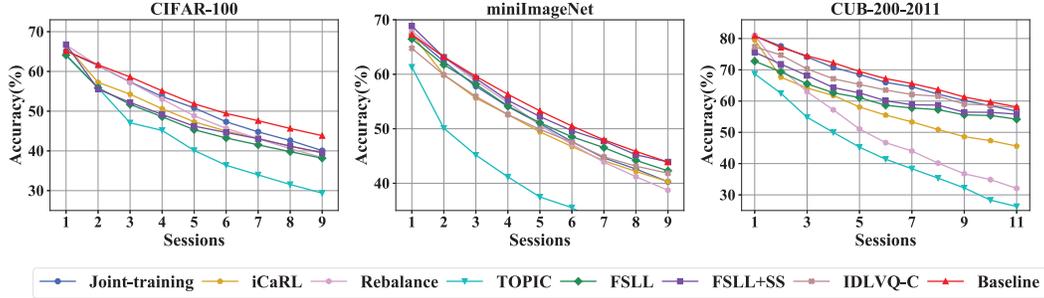}
  \caption{Comparison of the proposed baseline model with state-of-the-art IFL and IL methods and the joint-training method.The baseline model outperforms all the other methods.}
  \label{fig_table1}
\end{figure}

\section{Overcoming Catastrophic Forgetting in IFL by Finding Flat Minima}
The goal of IFL is to preserve the old while adapting to the new efficiently. The results and analysis in Sec.~\ref{sec:analysis} suggest that it might be ``a bit late'' to try to prevent catastrophic forgetting in the few-shot learning sessions ($t\geq 2$), which motivates us to consider this problem in the base training session. 

\textbf{Overview of our approach.} 
To overcome catastrophic forgetting in IFL, we propose to find a $b$-flat ($b>0$) local minima $\theta^\star$ of the base training objective function and then fine-tune the model within the flat region in later few-shot learning sessions.
Specifically, for any parameter vector $\theta$ in the flat region, i.e., $\theta^\star-b  \preceq\theta\preceq\ \theta^\star+b$, the risk (loss) of the base classes is minimized such that the classes are well separated in the embedding space of $f_\theta$. In the later incremental few-shot learning sessions ($t\geq 2$), we fine-tune the model parameters within this region to learn new classes, i.e., to find 
\begin{align*}
    \theta' = \arg\min_\theta \sum_{z\in \mathcal{D}^t}\mathcal{L}(z;\theta),~~ \text{s.t.}~~\ \theta^\star-b  \preceq\theta\preceq\ \theta^\star+b.
\end{align*}
As such, the fine-tuned model $\theta'$ can adapt to new classes while preserving the old ones. Also, due to the nature of few-shot learning, to avoid excessive training and overfitting, it suffices to tune the model in a relatively small region. A graphical illustration of our approach and prior arts, as well as the notions of sharp minima and flat minima, are presented in Fig.~\ref{fig_illustration}. 

\subsection{Searching for Flat Local Minima in the Base Training Stage}

A formal definition of $b$-flat local minima is given as follows.

\begin{mydef}[$b$-Flat Local Minima]\label{def:minima}
Given a real-valued objective function $\mathcal{L}(z;\theta)$, for any $b>0$, $\theta^\star$ is a $b$-flat local minima of $\mathcal{L}(z;\theta)$, if the following conditions are satisfied.
\begin{itemize}
    \item Condition 1: 
    $\mathcal{L}(z;\theta^\star)=\mathcal{L}(z;\theta^\star+\epsilon)$, where $-\mathbf{b} \preceq\epsilon\preceq \mathbf{b}$ and $\mathbf{b}_i=b$. 
    \item Condition 2: 
    there exist $\mathbf{c}_1 \prec \theta^\star-{b}$ and $\mathbf{c}_2 \succ \theta^\star+{b}$, s.t. $\mathcal{L}(z;\theta)>\mathcal{L}(z;\theta^\star)$, where $\mathbf{c}_1 \prec \theta \prec \theta^\star-{b}$ and  $\mathcal{L}(z;\theta^\star)<\mathcal{L}(z;\theta)$, where $\theta^\star+{b} \prec\theta \prec \mathbf{c}_2$.
\end{itemize}

\end{mydef}

\begin{figure}[t]
  \centering
    \includegraphics[width = 0.98\linewidth]{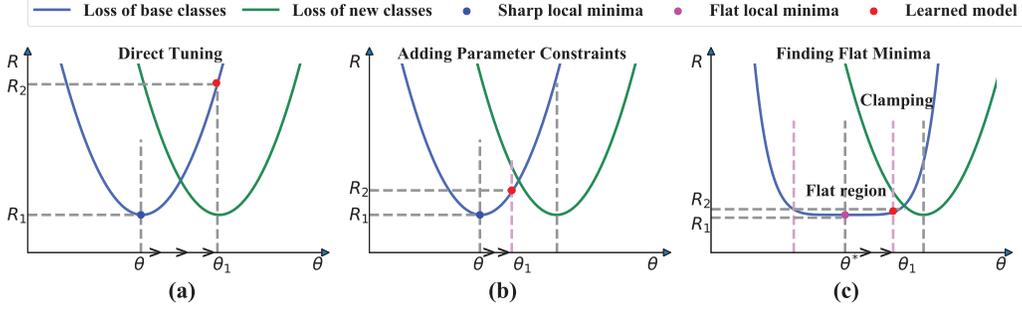} \label{fig_emb_tf}
  \caption{Illustration of our approach and existing solutions. $\to$ indicates the incremental learning steps on new classes. $R_1$ and $R_2$ respectively denote the loss of base classes before and after minimizing the loss of new classes. (a) SGD finds sharp minima in the base training. Directly tuning the model on new classes will result in a severe performance drop on base classes. (b) Enforcing strong constraints on parameters by penalizing parameter changes~\citep{aljundi2018memory,kirkpatrick2017overcoming,FSLL} may still lead to a significant performance drop on base classes. (c) Finding flat local minima of base classes and clamping the parameters after trained on new classes to make them fall within the flat region can effectively mitigate catastrophic forgetting.
  }
  \label{fig_illustration}
\end{figure}

In practice, it is hard to find the flat local minima that strictly satisfies the above definition, which may not even exist. Hence, our goal is to find an approximately flat local minima of the base training objective function. To this end, we propose to add some small random noise to the model parameters. The noise can be added for multiple times to obtain similar but different loss functions, which will be optimized together to locate the flat minima region. The intuition is clear -- the parameter vectors around the flat local minima also have small function values. 

To formally state the idea, we assume that the model is parameterized by $\theta =\{\phi,\psi\}$, where $\phi$ denotes the parameters of the embedding network and $\psi$ denotes the parameters of the classifier. $z$ denotes a labelled training sample. 
Denote the loss function by $\mathcal{L}$: $\mathbb{R}^{d_z}\to\mathbb{R}$. 
Our target is to minimize the expected loss function $R$: $\mathbb{R}^{d}\to\mathbb{R}$ w.r.t. the joint distribution of data $z$ and noise $\epsilon$, i.e.,
\begin{equation}\label{eq:Risk}
    R(\theta) = \int_{\mathbb{R}^{d_\epsilon}}\int_{\mathbb{R}^{d_z}}\mathcal{L}(z; \phi+\epsilon, \psi)\:d P(z)d P(\epsilon) = \mathbb{E}[\mathcal{L}(z; \phi+\epsilon, \psi)],
\end{equation}
where $P(z)$ is the data distribution and $P(\epsilon)$ is the noise distribution, and $z$ and $\epsilon$ are independent.
Since it is impossible to minimize the expected loss, we minimize its estimation, the empirical loss, which is given by
\begin{equation}\label{eq:emprical_risk}
    \mathcal{L}(\theta) = \frac{1}{M}\sum_{j=1}^{M}\mathcal{L}_{\text{base}}(z;\phi +\epsilon_j, \psi), ~\text{where}
\end{equation}

\begin{equation}\label{eq:base_loss}
    \mathcal{L}_{\text{base}}(z;\phi +\epsilon_j, \psi) =  \frac{1}{|\mathcal{D}^1|}\sum_{z\in\mathcal{D}^1}\mathcal{L}_{ce}(z;\phi +\epsilon_j, \psi) + \lambda\frac{1}{|\mathcal{C}^1|}\sum_{c\in\mathcal{C}^1}\|p_c - p_c^*\|_2^2,
\end{equation}
where $\epsilon_j$ is a noise vector sampled from $P(\epsilon)$, $M$ is the sampling times, $\mathcal{L}_{ce}(z;\phi +\epsilon_j, \psi)$ refers to the cross-entropy loss of a training sample $z$, and $p_c$ and $p_c^*$ are the class prototypes before and after injecting noise respectively. The first term of $\mathcal{L}_{base}$ is designed to find the flat region where the parameters $\phi$ of the embedding network can well separate the base classes. The second term enforces the class prototypes fixed within such region, which is designed to solve the prototype drift problem~\citep{Yu2020SemanticDC, IDLVQC} (the class prototypes change after updating the network) in later incremental learning sessions such that the saved base class prototypes can be directly used for evaluation in later sessions.

\subsection{Incremental Few-shot Learning within the Flat Region} \label{sec:incremental}
In the incremental few-shot learning sessions ($t\geq 2$), we fine-tune the parameters $\phi$ of the embedding network \textit{within the flat region} to learn new classes. It is worth noting that while the flat region might be relatively small, it is enough for incremental few-shot learning. Because  only few training samples are provided for each new class, to prevent overfitting in few-shot learning, excessive training should be avoided and only a small number of update iterations should be applied. 

We employ a metric-based classification algorithm with Euclidean distance to fine-tune the parameters. The loss function is defined as
\begin{equation}\label{eq:euclidean}
    \mathcal{L}_{m}(z;\phi) = - \sum_{z\in\mathcal{D}}\sum_{c\in\mathcal{C}}\mathbbm{1}(y=c)\log(\frac{e^{-d(p_{c},f(x;\phi))}}{\sum_{c_k\in \mathcal{C}}e^{-d(p_{c_k},f(x;\phi))}}),
z\end{equation}
where $d(\cdot,\cdot)$ denotes Euclidean distance, $p_c$ is the prototype of class $c$, $\mathcal{C}=\bigcup_{i=1}^t \mathcal{C}^{i}$ refers to all encountered classes, and $\mathcal{D}=\mathcal{D}^t\bigcup\mathcal{P}$ denotes the union of the current training data $\mathcal{D}^t$ and the exemplar set $\mathcal{P}=\{P_2,...,P_{t-1}\}$, where $P_{t_e} (2 \leq t_e<t)$ is the set of saved exemplars in session $t_e$. 
Note that the prototypes of new classes are computed by Eq.~\ref{eq:prototypes}, and those of base classes are saved in the base session. After updating the embedding network parameters, we clamp them to ensure that they fall within the flat region, i.e. $\phi^\star-b \preceq \phi \preceq \phi^\star +b$, where $\phi^\star$ denotes the optimal parameter vector learned in the base session. After fine-tuning, we evaluate the model using the nearest class mean classifier as in Eq.~\ref{eq:prototypes}, with previously saved prototypes and newly computed ones. The whole training process is described in Algorithm~\ref{alg:1}. Note that to calibrate the estimates of the classifier, we normalize all prototypes to make those of base classes and those of new classes have the same norm.

\input{algorithm1}

%% file: algorithm1.tex
\SetKwInput{KwInput}{Input} 
\SetKwInput{KwOutput}{Output}
\SetKwInput{KwInitialize}{Initialize}    
\begin{algorithm}[t]
\SetAlgoLined
\KwInput{the flat region bound $b$, randomly initialized $\theta=\{\phi,\psi\}$, the step sizes $\alpha$ and $\beta$.}
\tcp{\textcolor{blue}{Training over base classes $t=1$}}
\For{{\upshape epoch} k = {\upshape 1,2,...}}{
\For{j = {\upshape 1,2,..., $M$}}{
    Sample a noise vector $\epsilon_j\sim P(\epsilon)$, s.t. $-\mathbf{b} \preceq\epsilon_j\preceq \mathbf{b}$\;
    Add the noise to the parameters of the embedding network, i.e., $\theta = \{\phi + \epsilon_j, \psi\}$\;
    Compute the base loss $\mathcal{L}_{base}$ with Eq.~\ref{eq:base_loss}\;
    Reset the parameters, i.e., $\theta = \{\phi, \psi\}$\;
}
Update $\theta = \theta - \alpha\nabla\mathcal{L}(\theta)$ with the loss $\mathcal{L}$ defined in Eq.~\ref{eq:emprical_risk}.
}
Normalize and save the prototype of each base class\;
\BlankLine
\tcp{\textcolor{blue}{Incremental learning $t\geq 2$}}
Combine the training data $\mathcal{D}^t$ and the exemplars saved in previous few-shot sessions $2 \leq t_e<t$\;
\For{{\upshape epoch} k = {\upshape 1,2,...}}{
Compute the metric-based classification loss $\mathcal{L}_{m}$ by Eq.~\ref{eq:euclidean}\;
Update $\phi = \phi - \beta\nabla\mathcal{L}_{m}(z;\phi)$\;
Clamp the parameters $\phi$ to ensure they fall in the flat minima region\;
}
Randomly select and save a few exemplars from the training data $\mathcal{D}^t$\;
Normalize and save the prototype of each new class\;
\KwOutput{Model parameters $\theta=\{\phi,\psi\}$.}
\caption{F2M} \label{alg:1}
\end{algorithm}

%% file: Convergence_Analysis.tex
\subsection{Convergence Analysis}\label{sec:theory}
Our aim is to find a flat region within which all parameter vectors work well. We then minimize the expected loss w.r.t. the joint distribution of noise $\epsilon$ and data $z$. To approximate this expected loss, we sample from $P(\epsilon)$ for multiple times in each iteration and optimize the objective function using stochastic gradient descent (SGD). Here, we provide theoretical guarantees for our method. Given the non-convex loss function in Eq.~\ref{eq:emprical_risk}, we prove the convergence of our proposed method. The proof idea is inspired by the convergence analysis of SGD~\citep{bottou2018optimization,Kiefer1952StochasticEO}. 

Formally, in each batch $k$, let $z_k$ denote the batch data, $\{\epsilon_j\}_{j=1}^{M}$ be the sampled noises, and $\alpha_k$ be the step size. In the base training session, we update the model parameters as follows:
\begin{equation}\label{eq:base_update}
    \theta_{k+1} = \theta_{k} - \frac{\alpha_k}{M}\sum_{j=1}^{M}\nabla\mathcal{L}_{\text{base}}(z_k;\phi_k +\epsilon_j, \psi_k) = \theta_{k} - \frac{\alpha_k}{M}\sum_{j=1}^{M}g(z_k;\phi_k +\epsilon_j, \psi_k),
\end{equation}

where $g(z_k;\phi_k +\epsilon_j, \psi_k)=\nabla\mathcal{L}_{\text{base}}(z_k;\phi_k +\epsilon_j, \psi_k)$ is the gradient. To formally analyze the convergence of our algorithm, we define the following assumptions.

\begin{assumption}[L-smooth risk function] \label{assump:risk function}
The expected loss function $R: \mathbb{R}^d\to\mathbb{R}$ {\rm(Eq.~\ref{eq:Risk})} is continuously differentiable and L-smooth with constant $L>0$ such that \label{assump:L-smooth}
\begin{equation}
    \|\nabla R(\theta) - \nabla R(\theta')\|_2 \leq L\|\theta-\theta'|.
\end{equation}
\end{assumption}

This assumption is significant for the convergence analysis of gradient-based optimization algorithms, since it limits how fast the gradient of the loss function can change w.r.t. the parameter vector.

\begin{assumption}\label{assump:gradient}
The expected loss function satisfies the following conditions:

\begin{itemize}
    \item Condition 1: 
    $R$ is bounded below by a scalar $R^\star$, given the sequence of parameters $\{\theta_k\}$.
    \item Condition 2:
    For all $k\in \mathbb{N}$ and $j\in[1, M]$, 
    \begin{equation}
        \mathbb{E}_{z_{k},\epsilon_j}[g(z_k;\phi_k +\epsilon_j, \psi_k)] = \nabla R(\theta_k).
    \end{equation}
    \item Condition 3:
    There exist scalars $m_1\geq0$ and $m_2\geq0$, for all $k\in \mathbb{N}$ and $j\in[1, M]$,
    \begin{equation} \label{eq:var}
        \mathbb{V}_{z_{k},\epsilon_j}[g(z_k;\phi_k +\epsilon_j, \psi_k)] \leq m_1 + m_2\|\nabla R(\theta_k)\|_{2}^{2}.
    \end{equation}
\end{itemize}
\end{assumption}

$\mathbb{E}_{z_{k},\epsilon_j}[\cdot]$ denotes the expectation w.r.t. the joint distribution of random variables $z_k$ and $\epsilon_j$, and $\mathbb{V}_{z_{k},\epsilon_j}[\cdot]$ denotes the variance.
Condition 1 ensures that the expected loss $R$ is bounded by a minimum value $R^\star$ during the updates, which is a natural and practical assumption. Condition 2 assumes that the gradient $g(z_k;\phi_k +\epsilon_j, \psi_k)$ is an unbiased estimate of $\nabla R(\theta_k)$. This is a strict assumption made to simplify the proof, but it can be easily relaxed to a general and easily-met condition that there exist $\mu_1 \geq \mu_2>0$ satisfying $\|\mathbb{E}_{z_{k},\epsilon_j}[g(z_k;\phi_k +\epsilon_j, \psi_k)]\|_2 \leq \mu_1\|\nabla R(\theta_k)\|_2$ and $\nabla R(\theta_k)^{T}\mathbb{E}_{z_{k},\epsilon_j}[g(z_k;\phi_k +\epsilon_j, \psi_k)] \geq \mu_2\|\nabla R(\theta_k)\|_2^2$.
Therefore, the convergence can be proved in a similar way using the techniques presented in the Appendix. Condition 3 assumes the variance of the gradient $g(z_k;\phi_k +\epsilon_j, \psi_k)$ cannot be arbitrarily large, which is also reasonable in practice.
To facilitate later analysis, similar to~\citep{Robbins2007ASA}, we restrict the step sizes  as follows.

\begin{assumption}\label{assump:lr}
The learning rates satisfy:
\begin{equation}\label{eq:lr}
    \sum_{k=1}^{\infty}\alpha_k=\infty, \:\:\sum_{k=1}^{\infty}\alpha_k^2<\infty.
\end{equation}
\end{assumption}

This assumption can be easily met, since in practice the learning rate $\alpha_k$ is usually far less than $1$ and decreases w.r.t. $k$. Based on the above assumptions, we can derive the following theorem. 
\begin{theorem}\label{eq:theorem}
Under assumptions~\ref{assump:L-smooth}, \ref{assump:gradient} and \ref{assump:lr}, we further assume that the risk function $R$ is twice differentiable, and that $\|\nabla R(\theta)\|_2^2$ is $L_2$-smooth with constant $L_2>0$, then we have
\begin{equation}
    \lim_{k\to\infty} \mathbb{E}[\|\nabla R(\theta_k)\|_2^2] = 0.
\end{equation}
\end{theorem}
This theorem establishes the convergence of our algorithm. The proof is provided in Appendix~A.1.

%% file: Experiment.tex
\section{Experiments}\label{sec:exp}
In this section, we empirically evaluate our proposed method for incremental few-shot learning and demonstrate its effectiveness by comparison with state-of-the-art methods.

\subsection{Experimental Setup}\label{sec:exp_setup}

\input{Table2}
\input{Table1}
\input{Table3}

\textbf{Datasets.} For CIFAR-100 and \textit{mini}ImageNet, we randomly select 60 classes as the base classes and the remaining 40 classes as the new classes. In each incremental learning session, we construct 5-way 5-shot tasks by randomly picking 5 classes and sampling 5 examples for each class. For CUB-200-2011 with $200$ classes, we select $100$ classes as the base classes and $100$ classes as the new ones. We test 10-way 5-shot tasks on this dataset.

\textbf{Baselines.}
We compare our method F2M with 8 methods: the Baseline proposed in Sec.~\ref{sec:analysis}, a joint-training method that uses all previously seen data including the base and the following few-shot tasks for training, the classifier re-training method (cRT)~\cite{cRT} for long-tailed classification trained with all encountered data, iCaRL~\cite{icarl}, Rebalance~\cite{NCM}, TOPIC~\cite{TOPIC}, FSLL~\cite{FSLL}, and IDLVQ-C~\cite{IDLVQC}. For a fair comparison, we re-implement cRT~\cite{cRT}, iCaRL~\cite{icarl}, Rebalance~\cite{NCM}, FSLL~\cite{FSLL}, and the joint-training method and tune them to their best performance. We also provide the results reported in the original papers for comparison. The results of TOPIC~\cite{TOPIC} and IDLVQ-C~\cite{IDLVQC} are copied from the original papers.
Note that for IL, joint-training is naturally the upper bound of incremental learning algorithms, however, for IFL, joint-training is not a good approximation of the upper bound 
because data imbalance makes the model perform significantly poorer on new classes (long-tailed classes).
To address the data imbalance issue, we re-implement 
the cRT method as the \textit{approximate upper bound}.

\textbf{Experimental details.} The experiments are conducted with NVIDIA GPU RTX3090 on CUDA 11.0. We randomly split each dataset into multiple tasks (sessions). For each dataset (with a fixed split), we run each algorithm for 10 times
and report the mean accuracy.
We adopt ResNet18~\cite{he2016deep} as the backbone network. 
For data augmentation, we use standard random crop and horizontal flip. In the base training stage, we select the last 4 or 8 convolution layers to inject noise, because these layers output higher-level feature representations. The flat region bound $b$ is set as 0.01. We set the number of times for noise sampling as $M= 2\sim4$, since a larger $M$ will increase the training time.
In each incremental few-shot learning session, the total number of training epochs is 6, and the learning rate is 0.02. To verify the correctness of our implementation, we conduct experiments on incremental learning and compare our results to those reported on CIFAR-100 in Appendix~A.3. More experiment details are provided in Appendix~A.2. 

\subsection{Comparison with the State-of-the-Art}
\textbf{F2M outperforms the state-of-the-art methods}. The main results on CIFAR-100, miniImageNet and CUB-200-2011 are presented in Table~\ref{tab:cifar}, Table~\ref{tab_miniImageNet} and Table~\ref{tab_cub} respectively. Based on the experiment results, we have the following observations: \textbf{1)} The Baseline introduced in Sec. \ref{sec:analysis} outperforms the state-of-the-art approaches on all incremental sessions. \textbf{2)} As expected, cRT consistently outperforms the Baseline up to $1$\% to $3$\% by considering the data imbalance problem and applying proper techniques to tackle the long-tailed classification problem to improve performance. Hence, it is reasonable to use cRT as the approximate upper bound of IFL. \textbf{3)} Our F2M outperforms the state-of-the-art methods and the Baseline. Moreover, the performance of F2M is very close to the approximate upper bound, i.e., the gap with cRT is only $0.2$\% in the last session on \emph{mini}ImageNet.
The results show that even with strong constraints  \cite{NCM,icarl,FSLL} and saved examplars of base classes \cite{NCM,icarl,IDLVQC}, current methods cannot effectively address the catastrophic forgetting problem. In contrast, finding flat minima seems a promising approach to overcome this harsh problem.

\subsection{Ablation Study and Analysis} \label{subsec:ablation}
\input{Table_flatness}

\textbf{Analysis on the flatness of local minima.} Here, we verify that our method can find a more flat local minima than the Baseline. For a found local minima $\theta^\star$, we measure its flatness as follows. We sample the noise for $1000$ times. For each time, we inject the sampled noise to $\theta^\star$ and calculate the loss $\mathcal{L}_i$. Then, we adopt the indicator $I=\frac{1}{1000}\sum_{i=1}^{1000}(\mathcal{L}_i-\mathcal{L}^{*})^2$ and variance $\sigma^2 = \frac{1}{1000}\sum_{i=1}^{1000}(\mathcal{L}_i-\overline{\mathcal{L}})^2$ to measure the flatness. $\mathcal{L}^{*}$ denotes the loss of $\theta^\star$, and $\overline{\mathcal{L}}$ denotes the average loss of $\{\mathcal{L}_i\}_{i=1}^{1000}$. The values of the indicator and variance of F2M and the Baseline are presented in Table~\ref{tab:flatness}, which clearly demonstrate that our method can find a more flat local minima. 

\input{Table4}

\textbf{Ablation study on the designs of our method.}
Here, we study the effectiveness of each design of our method, including adding noise to the model parameters for finding $b$-flat local minima ({FM}) during the base training session, the prototype fixing term ({PF}) used in the base training objective (Eq.~\ref{eq:base_loss}), parameter clamping ({PC}) during incremental learning, and prototype normalization ({PN}).
We conduct an ablation study by removing each component in turn and report the experimental results in Table~\ref{table:ablation}.

\noindent
\emph{Finding $b$-flat local minima.} 
Standard supervised training with SGD as the optimizer tends to converge to a sharp local minima. It leads to a significant drop in performance because the loss changes quickly in the neighborhood of the sharp local minima. As shown in Table~\ref{table:ablation}, even with parameter clamping during incremental learning, 
the performance still drops significantly.
In contrast, restricting the parameters in a small flat region can mitigate the forgetting problem.

\noindent
\emph{Prototype fixing.}
Without fixing the prototypes after injecting noise to selected layers during the process of finding local minima, i.e. removing the second term of Eq.~\ref{eq:base_loss}, it is still possible to tune the model within the flat region to well separate base classes. However, the saved prototypes of base classes will become less accurate because 
the embeddings of the base samples suffer from semantic drift~\cite{Yu2020SemanticDC}. 
As shown in Table~\ref{table:ablation}, it results in a performance drop of nearly $0.6$\%.

\noindent
\emph{Parameter clamping.}
Parameter clamping restricts the model parameters to the ${b}$-flat region after incremental few-shot learning. Outside the ${b}$-flat region, the performance drops quickly. It can be seen from Table~\ref{table:ablation} that removing parameter clamping leads to a significant drop in performance.

\noindent
\emph{Prototype normalization.} 
As mentioned in Sec.~\ref{sec:incremental}, we normalize the class prototypes to calibrate the estimates of the class mean classifier. 
The results in Table~\ref{table:ablation} show the effectiveness of normalization, which helps to further improve the performance.

\textbf{Study of the flat region bound $b$.} We study the effect of the flat region bound $b$ for 5-way 5-shot incremental learning on CIFAR-100. We report the test accuracy in session 1 (base session) and session 9 (last session) w.r.t. different $b$ in Table~\ref{tab:bound}. It can be seen that the best results are achieved for $b\in[0.005, 0.02]$. A larger $b$ (e.g., 0.04 or 0.08) leads to a significant performance drop on base classes, even for those in session 1, indicating that there may not exist a large flat region around a good local minima. Meanwhile, a smaller $b$ (e.g., 0.0025) results in a performance decline on new classes, due to the overly small capacity of the flat region. This illustrates the trade-off effect of $b$.

\input{Table_b}

%% file: Table2.tex
\begin{table*}[t]
\renewcommand\arraystretch{1.0}
\caption{Classification accuracy on CIFAR-100 for 5-way 5-shot incremental learning. \textsuperscript{$\ast$} indicates our re-implementation.} 
\label{tab:cifar}
\centering
\resizebox{0.9\textwidth}{!}{
\begin{tabular}{lcccccccccc}
  \toprule
  \multirow[m]{2}{*}{\textbf{Method}} & \multicolumn{9}{c}{\textbf{sessions}} & \multirow[m]{2}{*}{\shortstack{\textbf{The gap}\\ \textbf{with cRT}}}\\
  \cmidrule[0.05em]{2-10}
  & 1 & 2 & 3 & 4 & 5 & 6 & 7 & 8 & 9 &\\
  \midrule
  cRT~\cite{cRT}\textsuperscript{$\ast$}            & 65.18 & 63.89 & 60.20 & 57.23 & 53.71 & 50.39 & 48.77 & 47.29 & 45.28 & -\\
  Joint-training\textsuperscript{$\ast$}               & 65.18 & 61.45 & 57.36 & 53.68 & 50.84 & 47.33 & 44.79 & 42.62 & 40.08 & -5.20\\
  Baseline                                          & 65.18 & 61.67 & 58.61 & 55.11 & 51.86 & 49.43 & 47.60 & 45.64 & 43.83 & -1.45\\
  \midrule
  iCaRL~\cite{icarl}\textsuperscript{$\ast$}        & 66.52 & 57.26 & 54.27 & 50.62 & 47.33 & 44.99 & 43.14 & 41.16 & 39.49 & -5.79\\
  Rebalance~\cite{NCM}\textsuperscript{$\ast$}      & \textbf{66.66} & 61.42 & 57.29 & 53.02 & 48.85 & 45.68 & 43.06 & 40.56 & 38.35 & -6.93 \\
  FSLL~\cite{FSLL}\textsuperscript{$\ast$}          & 65.18 & 56.24 & 54.55 & 51.61 & 49.11 & 47.27 & 45.35 & 43.95 & 42.22 & -3.08\\
  iCaRL~\cite{icarl}     & 64.10 & 53.28 & 41.69 & 34.13 & 27.93 & 25.06 & 20.41 & 15.48 & 13.73 & -31.55\\
  Rebalance~\cite{NCM}   & 64.10 & 53.05 & 43.96 & 36.97 & 31.61 & 26.73 & 21.23 & 16.78 & 13.54 & -31.74\\
  TOPIC~\cite{TOPIC}     & 64.10 & 55.88 & 47.07 & 45.16 & 40.11 & 36.38 & 33.96 & 31.55 & 29.37 & -15.91 \\
  FSLL~\cite{FSLL}       & 64.10 & 55.85 & 51.71 & 48.59 & 45.34 & 43.25 & 41.52 & 39.81 & 38.16 & -7.12\\
  FSLL+SS~\cite{FSLL}    & 66.76 & 55.52 & 52.20 & 49.17 & 46.23 & 44.64 & 43.07 & 41.20 & 39.57 & -5.71\\
  \midrule
  \textbf{F2M}                   & 64.71 & \textbf{62.05} & \textbf{59.01} & \textbf{55.58} & \textbf{52.55} & \textbf{49.96} & \textbf{48.08} & \textbf{46.28} & \textbf{44.67} & \textbf{-0.61}\\
  \bottomrule
\end{tabular}
}
\end{table*}

%% file: Table1.tex
\begin{table*}[t]
\renewcommand\arraystretch{1.0}
\caption{Classification accuracy on \emph{mini}ImageNet for 5-way 5-shot incremental learning. \textsuperscript{$\ast$} indicates our re-implementation.} 
\label{tab_miniImageNet}
\centering
\resizebox{0.9\textwidth}{!}{
\begin{tabular}{lcccccccccc}
  \toprule
  \multirow[m]{2}{*}{\textbf{Method}} & \multicolumn{9}{c}{\textbf{sessions}} & \multirow[m]{2}{*}{\shortstack{\textbf{The gap}\\ \textbf{with cRT}}}\\
  \cmidrule[0.05em]{2-10}
  & 1 & 2 & 3 & 4 & 5 & 6 & 7 & 8 & 9 &\\
  \midrule
  cRT~\cite{cRT}\textsuperscript{$\ast$}            & 67.30 & 64.15 & 60.59 & 57.32 & 54.22 & 51.43 & 48.92 & 46.78 & 44.85 & -\\
  Joint-training\textsuperscript{$\ast$}            & 67.30 & 62.34 & 57.79 & 54.08 & 50.93 & 47.65 & 44.64 & 42.61 & 40.29 & -4.56\\
  Baseline                                          & 67.30 & 63.18 & 59.62 & 56.33 & 53.28 & 50.50 & 47.96 & 45.85 & 43.88 & -0.97\\
  \midrule
  iCaRL~\cite{icarl}\textsuperscript{$\ast$}        & 67.35 & 59.91 & 55.64 & 52.60 & 49.43 & 46.73 & 44.13 & 42.17 & 40.29 & -4.56\\
  Rebalance~\cite{NCM}\textsuperscript{$\ast$}      & 67.91 & 63.11 & 58.75 & 54.83 & 50.68 & 47.11 & 43.88 & 41.19 & 38.72 & -6.13 \\
  FSLL~\cite{FSLL}\textsuperscript{$\ast$}          & 67.30 & 59.81 & 57.26 & 54.57 & 52.05 & 49.42 & 46.95 & 44.94 & 42.87 & -1.11\\
  iCaRL~\cite{icarl}     & 61.31 & 46.32 & 42.94 & 37.63 & 30.49 & 24.00 & 20.89 & 18.80 & 17.21 & -27.64\\
  Rebalance~\cite{NCM}   & 61.31 & 47.80 & 39.31 & 31.91 & 25.68 & 21.35 & 18.67 & 17.24 & 14.17 & -30.68\\
  TOPIC~\cite{TOPIC}     & 61.31 & 50.09 & 45.17 & 41.16 & 37.48 & 35.52 & 32.19 & 29.46 & 24.42 & -20.43 \\
  FSLL~\cite{FSLL}       & 66.48 & 61.75 & 58.16 & 54.16 & 51.10 & 48.53 & 46.54 & 44.20 & 42.28 & -2.57\\
  FSLL+SS~\cite{FSLL}    & \textbf{68.85} & 63.14 & 59.24 & 55.23 & 52.24 & 49.65 & 47.74 & 45.23 & 43.92 & -0.93\\
  IDLVQ-C~\cite{IDLVQC}  & 64.77 & 59.87 & 55.93 & 52.62 & 49.88 & 47.55 & 44.83 & 43.14 & 41.84 & -3.01\\
  \midrule
  \textbf{F2M}           & 67.28 & \textbf{63.80} & \textbf{60.38} & \textbf{57.06} & \textbf{54.08} & \textbf{51.39} & \textbf{48.82} & \textbf{46.58} & \textbf{44.65} & \textbf{-0.20}\\
  \bottomrule
\end{tabular}
}
\end{table*}

%% file: Table3.tex
\begin{table*}[h]
\renewcommand\arraystretch{1.0}
\caption{Classification accuracy on CUB-200-2011 for 10-way 5-shot incremental learning.\textsuperscript{$\ast$} indicates our re-implementation.} 
\label{tab_cub}
\centering
\resizebox{\textwidth}{!}{
\begin{tabular}{lcccccccccccc}
  \toprule
  \multirow[m]{2}{*}{\textbf{Method}} & \multicolumn{11}{c}{\textbf{sessions}} & \multirow[m]{2}{*}{\shortstack{\textbf{The gap}\\ \textbf{with cRT}}}\\
  \cmidrule[0.05em]{2-12}
  & 1 & 2 & 3 & 4 & 5 & 6 & 7 & 8 & 9 & 10 & 11\\
  \midrule
  cRT~\cite{cRT}\textsuperscript{$\ast$}      & 80.83 & 78.51 & 76.12 & 73.93 & 71.46 & 68.96 & 67.73 & 66.75 & 64.22 & 62.53 & 61.08 & -\\
  Joint-training\textsuperscript{$\ast$}         & 80.83 & 77.57 & 74.11 & 70.75 & 68.52 & 65.97 & 64.58 & 62.22 & 60.18 & 58.49 & 56.78 & -4.30 \\
  Baseline                                    & 80.87 & 77.15 & 74.46 & 72.26 & 69.47 & 67.18 & 65.62 & 63.68 & 61.30 & 59.72 & 58.12 & -2.96\\
  \midrule
  iCaRL~\cite{icarl}\textsuperscript{$\ast$}  & 79.58 & 67.63 & 64.17 & 61.80 & 58.10 & 55.51 & 53.34 & 50.89 & 48.62 & 47.34 & 45.60 & -15.48\\
  Rebalance~\cite{NCM}\textsuperscript{$\ast$}& 80.94 & 70.32 & 62.96 & 57.19 & 51.06 & 46.70 & 44.03 & 40.15 & 36.75 & 34.88 & 32.09 & -28.99\\
  FSLL~\cite{FSLL}\textsuperscript{$\ast$}    & 80.83 & 77.38 & 72.37 & 71.84 & 67.51 & 65.30 & 63.75 & 61.16 & 59.05 & 58.03 & 55.82 & -5.26\\
  iCaRL~\cite{icarl}     & 68.68 & 52.65 & 48.61 & 44.16 & 36.62 & 29.52 & 27.83 & 26.26 & 24.01 & 23.89 & 21.16 & -39.92\\
  Rebalance~\cite{NCM}   & 68.68 & 57.12 & 44.21 & 28.78 & 26.71 & 25.66 & 24.62 & 21.52 & 20.12 & 20.06 & 19.87 & -41.21\\
  TOPIC~\cite{TOPIC}     & 68.68 & 62.49 & 54.81 & 49.99 & 45.25 & 41.40 & 38.35 & 35.36 & 32.22 & 28.31 & 26.28 & -34.80\\
  FSLL~\cite{FSLL}       & 72.77 & 69.33 & 65.51 & 62.66 & 61.10 & 58.65 & 57.78 & 57.26 & 55.59 & 55.39 & 54.21 & -6.87\\
  FSLL+SS~\cite{FSLL}    & 75.63 & 71.81 & 68.16 & 64.32 & 62.61 & 60.10 & 58.82 & 58.70 & 56.45 & 56.41 & 55.82 & -5.26\\
  IDLVQ-C~\cite{IDLVQC}  & 77.37 & 74.72 & 70.28 & 67.13 & 65.34 & 63.52 & 62.10 & 61.54 & 59.04 & 58.68 & 57.81 & -3.27\\
  \midrule
  \textbf{F2M}                    & \textbf{81.07} & \textbf{78.16} & \textbf{75.57} & \textbf{72.89} & \textbf{70.86} & \textbf{68.17} & \textbf{67.01} & \textbf{65.26} & \textbf{63.36} & \textbf{61.76} & \textbf{60.26} & \textbf{-0.82}\\
  \bottomrule
\end{tabular}
}
\end{table*}

%% file: Table_flatness.tex
\begin{table*}[t]
\renewcommand\arraystretch{1.0}
\caption{Comparison of the flatness of the local minima found by the Baseline and our F2M.} 
\label{tab:flatness}
\centering
\resizebox{0.7\textwidth}{!}{
\begin{tabular}{lccccc}
  \toprule
  \multirow[m]{2}{*}{\textbf{Method}} & \multicolumn{2}{c}{Indicator $I$} & \multicolumn{2}{c}{Variance $\sigma^2$} \\
  \cmidrule[0.05em]{2-6}
  & Training Set & Testing Set & Training Set & Testing Set&\\
  \midrule
  Baseline       & 0.2993 & 0.4582 & 0.1451 & 0.2395 \\
  \textbf{F2M}   & \textbf{0.0506} & \textbf{0.0800} & \textbf{0.0296} & \textbf{0.0334} \\
  \bottomrule
\end{tabular}
}
\end{table*}

%% file: Table4.tex
\begin{table*}[t]
\renewcommand\arraystretch{1.0}
\caption{Ablation study of our F2M on CIFAR-100. 
{PD} refers to the performance dropping rate.} 
\label{table:ablation}
\centering
\resizebox{0.92\textwidth}{!}{
\begin{tabular}{cccccccccccccc}
  \toprule
  \multirow[m]{2}{*}{\textbf{FM}} & \multirow[m]{2}{*}{\textbf{PF}} & \multirow[m]{2}{*}{\textbf{PC}} & \multirow[m]{2}{*}{\textbf{PN}} & \multicolumn{9}{c}{\textbf{sessions}} & \multirow[m]{2}{*}{\textbf{PD}
  $\downarrow$}\\
  \cmidrule[0.05em]{5-13}
  & & & & 1 & 2 & 3 & 4 & 5 & 6 & 7 & 8 & 9 & \\
  \midrule
                &            &            &            & \textbf{65.18} & 60.83 & 53.13 & 43.57 & 23.75 & 10.76 & 08.26 & 07.24 & 06.45 & 58.73\\
                &            & \checkmark &            & \textbf{65.18} & 59.48 & 56.77 & 52.99 & 50.09 & 47.80 & 45.92 & 44.20 & 42.55 & 22.63\\    
     \checkmark & \checkmark &            & \checkmark & 64.71 & 59.54 & 53.03 & 45.09 & 41.68 & 39.04 & 38.64 & 37.19 & 36.01 & 28.70 \\
     \checkmark &            & \checkmark & \checkmark & 64.55 & 61.27 & 58.33 & 54.82 & 51.60 & 49.22 & 47.48 & 45.78 & 44.08 & 20.47\\
     \checkmark & \checkmark & \checkmark &            & 64.71 & 61.75 & 58.80 & 55.33 & 52.27 & 49.75 & 47.72 & 46.01 & 44.43 & 20.28\\
  \midrule
     \checkmark & \checkmark & \checkmark & \checkmark & 64.71 & \textbf{61.99} & \textbf{58.99} & \textbf{55.58} & \textbf{52.55} & \textbf{49.96} & \textbf{48.08} & \textbf{46.28} & \textbf{44.67} & \textbf{20.04}\\
  \bottomrule
\end{tabular}
}
\end{table*}

%% file: Table_b.tex
\begin{table*}[t]
\renewcommand\arraystretch{1.0}
\caption{Study of the flat region bound $b$ for 5-way 5-shot incremental learning on CIFAR-100. The top 3 results in each row are in boldface.} 
\label{tab:bound}
\centering
\resizebox{0.8\textwidth}{!}{
\begin{tabular}{lcccccc}
  \toprule
  \multirow[m]{2}{*}{\textbf{Session}} & \multicolumn{6}{c}{The hyperparameter $b$} \\
  \cmidrule[0.05em]{2-7}
  & 0.0025 & 0.005 & 0.01 & 0.02 & 0.04 & 0.08\\
  \midrule
  Session 1 (60 bases classes)       & \textbf{64.85} & 64.67 & \textbf{64.81} & \textbf{64.71} & 63.30 & 62.25 \\
  Session 9 (All 100 classes)            & 44.16 & \textbf{44.54} & \textbf{44.58} & \textbf{44.67} & 43.75 & 43.04 \\
  Session 9 (60 base classes)        & \textbf{59.58} & \textbf{59.69} & \textbf{59.73} & 59.44 & 58.38 & 57.21 \\
  Session 9 (40 new classes)         & 21.03 & \textbf{21.81} & \textbf{21.86} & \textbf{22.52} & 21.80 & 21.77 \\
  \bottomrule
\end{tabular}
}
\end{table*}

%% file: Conclusion.tex
\section{Conclusion}\label{sec:conclusion}
 We have proposed a novel approach to overcome catastrophic forgetting in incremental few-shot learning by finding flat local minima of the objective function in the base training stage and then fine-tuning the model within the flat region on new tasks. Extensive experiments on benchmark datasets show that our model can effectively mitigate catastrophic forgetting and adapt to new classes.
 A limitation of our method is that it may not be suitable for medium- or high-shot tasks, since the flat region is relatively small, which limits the model capacity. However, it is still possible to adapt our core idea for incremental learning. For example, one can search for a less flat but wider local minima region in the base training stage and tune the model within this region during incremental learning sessions, where previous techniques such as elastic weight consolidation (EWC)~\citep{kirkpatrick2017overcoming} can be used to constraint the model parameters. This could be an interesting direction for future research.

%% file: Appendix.tex
\section{Appendix}
\subsection{Proof of Theorem 4.1} \label{appendix:proof}
\begin{lemma}\label{lemma:pre}
By Assumption~\ref{assump:L-smooth} and \ref{assump:gradient}, we have
\begin{equation}\label{eq:lemma}
    \mathbb{E}_{z_k,\epsilon_{j}}[R(\theta_{k+1})] - R(\theta_k) \leq -\alpha_k\frac{2M-\alpha_k L(m_2+M)}{2M}\|\nabla R(\theta_k)\|_2^2  + \frac{\alpha_k^2L m_1}{2M}.
\end{equation}
\end{lemma}

\begin{proof}
By Assumption~\ref{assump:L-smooth}, an important consequence is that for all $\{\theta, \theta'\} \subset \mathbb{R}^d$, it satisfies that
\begin{equation}
    R(\theta) \leq R(\theta’) + \nabla R(\theta')^T(\theta-\theta') + \frac{1}{2}L\|\theta-\theta'\|_2^2. 
\end{equation}
Taken together, the above inequality and the parameter update equation (Eq.~\ref{eq:base_update}), it yields 
\begin{equation}\label{eq:substact}
    R(\theta_{k+1}) - R(\theta_k)  \leq \nabla R(\theta_k)^{T}(\theta_{k+1}-\theta_k) + \frac{1}{2}L\|\theta_{k+1} - \theta_{k}\|_2^{2} \\
                                   \leq - \alpha_k\nabla R(\theta_k)^{T}\overline{g} + \frac{\alpha_k^2 L}{2}\|\overline{g}\|_2^2,
\end{equation}

where $\overline{g}=\frac{1}{M}\sum_{j=1}^{M}g(z_k;\phi_k +\epsilon_j, \psi_k)$. Taking expectation on both sides of Eq.~\ref{eq:substact}, it yields
\begin{equation}\label{eq:E_substact}
    \mathbb{E}_{z_k,\epsilon_{j}}[R(\theta_{k+1})] - R(\theta_k) \leq - \alpha_k\nabla R(\theta_k)^{T}\mathbb{E}_{z_k,\epsilon_{j}}[\overline{g}] + \frac{\alpha_k^2 L}{2}\mathbb{E}_{z_k,\epsilon_{j}}[\|\overline{g}\|_2^2].
\end{equation}
$\mathbb{E}_{z_{k},\epsilon_j}[\cdot]$ denotes the expectation w.r.t. the joint distribution of random variables $z_k$ and $\epsilon_j$ given $\theta_k$. Note that $\theta_{k+1}$ (not $\theta_k$) depends on $z_k$ and $\epsilon_j$. Under Condition 2 of Assumption~\ref{assump:gradient}, the expectation of $\overline{g}$ satisfies that
\begin{equation}\label{eq:E}
    \mathbb{E}_{z_k,\epsilon_{j}}[\overline{g}] = \frac{1}{M}\sum_{j=1}^{M}\mathbb{E}_{z_k,\epsilon_{j}}[g(z_k;\phi_k +\epsilon_j, \psi_k)] = \nabla R(\theta_k).
\end{equation}
Assume that we sample the noise vector $\epsilon_j$ from $P(\epsilon)$ \emph{without replacement}. Under Condition~3 of Assumption~\ref{assump:gradient}, we have (see [1, p. 183])
\begin{equation}\label{eq:V}
    \mathbb{V}_{z_{k},\epsilon_j}[\overline{g}] \leq \frac{ \mathbb{V}_{z_{k},\epsilon_j}[g(z_k;\phi_k +\epsilon_j, \psi_k)]}{M} \leq \frac{m_1}{M} + \frac{m_2}{M}\|\nabla R(\theta_k)\|_2^2. 
\end{equation}
Taken together,  Eq.~\ref{eq:E} and Eq.~\ref{eq:V}, one obtains
\begin{equation}\label{eq:E2}
    \mathbb{E}_{z_k,\epsilon_{j}}[\|\overline{g}\|_2^2] = \mathbb{V}_{z_{k},\epsilon_j}[\overline{g}] + \|\mathbb{E}_{z_k,\epsilon_{j}}[\overline{g}]\|_2^2 \leq \frac{m_1}{M} + \frac{m_2+M}{M}\|\nabla R(\theta_k)\|_2^2. 
\end{equation}
Therefore, by Eq.~\ref{eq:E_substact}, \ref{eq:E} and \ref{eq:E2}, it yields
\begin{equation}
    \mathbb{E}_{z_k,\epsilon_{j}}[R(\theta_{k+1})] - R(\theta_k) \leq -\alpha_k\frac{2M-\alpha_k L(m_2+M)}{2M}\|\nabla R(\theta_k)\|_2^2  + \frac{\alpha_k^2L m_1}{2M}.
\end{equation}
\end{proof}

\begin{lemma}\label{lemma:liminf}
By Assumption~\ref{assump:L-smooth},~\ref{assump:gradient} and~\ref{assump:lr}, we have
\begin{equation}
    \liminf_{k\to\infty}\mathbb{E}[\|\nabla R(\theta_k)\|_2^2]=0.
\end{equation}
\end{lemma}

\begin{proof}
The first condition in Assumption~\ref{assump:lr} ensures that $\lim_{k\to\infty}\alpha_k=0$. Without loss of generality, we assume that for any $k\in\mathbb{N}$, $\alpha_k L(m_2+M) \leq M$. Denote by $\mathbb{E}[\cdot]$ the \emph{total expectation} w.r.t. all involved random variables. For example, $\theta_k$ is determined by the set of random variables $\{z_0, z_1,...,z_{k-1}, \epsilon_0, \epsilon_1,...,\epsilon_{k-1}\}$, and therefore the \emph{total expectation} of $R(\theta_k)$ is given by
\begin{equation}
    \mathbb{E}[R(\theta_k)]=\mathbb{E}_{z_0,\epsilon_0}\mathbb{E}_{z_1,\epsilon_1}...\mathbb{E}_{z_{k-1},\epsilon_{k-1}}[R(\theta_k)].
\end{equation}

Taking total expectation 
on both sides of Eq.\ref{eq:lemma}, we have
\begin{equation}
    \mathbb{E}[R(\theta_k+1)]-\mathbb{E}[R(\theta_k)] \leq -\frac{\alpha_k}{2}\mathbb{E}[\|\nabla R(\theta_k)\|_2^2] + \frac{\alpha_k^2L m_1}{2M}.
\end{equation}
For $k=0,1,2,...,K$, summing both sides of this inequality yields
\begin{equation}
    R^\star - \mathbb{E}[R(\theta_1)] \leq \mathbb{E}[R(\theta_{K+1})] - \mathbb{E}[R(\theta_{0})] \leq -\frac{1}{2}\sum_{k=0}^K\alpha_k\mathbb{E}[\|\nabla R(\theta_k)\|_2^2] + \frac{L m_1}{2M}\sum_{k=0}^K \alpha_k^2,
\end{equation}
where $R^\star$ is the lower bound in Condition 1 of Assumption~\ref{assump:gradient}. Rearranging the term gives
\begin{equation}
    \sum_{k=0}^K\alpha_k\mathbb{E}[\|\nabla R(\theta_k)\|_2^2] \leq 2(\mathbb{E}[R(\theta_1)]-R^\star) + \frac{L m_1}{M}\sum_{k=0}^K \alpha_k^2.
\end{equation}
By the second condition of Assumption~\ref{assump:lr}, we have
\begin{equation}\label{eq:infty}
    \lim_{K\to\infty}\mathbb{E}[\sum_{k=0}^K\alpha_k\|\nabla R(\theta_k)\|_2^2] \leq 2(\mathbb{E}[R(\theta_0)]-R^\star) + \lim_{K\to\infty}\frac{L m_1}{M}\sum_{k=0}^K \alpha_k^2 < \infty.
\end{equation}
Dividing both sides of Eq.~\ref{eq:infty} by $\sum_{k=1}^K\alpha_k$ and by the first condition of Assumption~\ref{assump:lr}, we have
\begin{equation}
    \lim_{K\to\infty}\mathbb{E}[\frac{\sum_{k=1}^K\alpha_k\|\nabla R(\theta_k)\|_2^2}{\sum_{k=1}^K\alpha_k}] = 0.
\end{equation}
The left-hand term of this equation is the weighed average of $\|\nabla R(\theta_k)\|_2^2$, and $\{\alpha_k\}$ are the weights. Hence, a direct consequence of this equation is that $\|\nabla R(\theta_k)\|_2^2$ cannot asymptotically stay far from zero, i.e.
\begin{equation}\label{eq:liminf}
    \liminf_{k\to\infty}\mathbb{E}[\|\nabla R(\theta_k)\|_2^2]=0.
\end{equation}
\end{proof}

We now prove {Theorem~\ref{eq:theorem}}, which is a stronger consequence than Lemma~\ref{lemma:liminf}. 
\begin{customthm}{4.1}\label{eq:theorem_a}
Under assumptions~\ref{assump:L-smooth}, \ref{assump:gradient} and \ref{assump:lr}, we further assume that the risk function $R$ is twice differentiable, and that $\|\nabla R(\theta)\|_2^2$ is $L_2$-smooth with constant $L_2>0$, then we have
\begin{equation}
    \lim_{k\to\infty} \mathbb{E}[\|\nabla R(\theta_k)\|_2^2] = 0.
\end{equation}
\end{customthm}
\begin{proof}
Define $F(\theta)\coloneqq\|R(\theta)\|_2^2$, then we have
\begin{flalign}\label{eq:substact_F}
&\begin{aligned}
    \mathbb{E}_{z_k,\epsilon_{j}}[F(\theta_{k+1})] - F(\theta_k) 
            & \leq \nabla F(\theta_k)^{T}\mathbb{E}_{z_k,\epsilon_{j}}[(\theta_{k+1}-\theta_k)] + \frac{1}{2}L_2\mathbb{E}_{z_k,\epsilon_{j}}[\|\theta_{k+1} - \theta_{k}\|_2^{2}] \\
            & \leq - \alpha_k\nabla F(\theta_k)^{T}\mathbb{E}_{z_k,\epsilon_{j}}[\overline{g}] + \frac{\alpha_k^2 L_2}{2}\mathbb{E}_{z_k,\epsilon_{j}}[\|\overline{g}\|_2^2] \\
            & \leq -2\alpha_k\nabla R(\theta_k)^{T} \nabla^{2}R(\theta_k)^{T}\mathbb{E}_{z_k,\epsilon_{j}}[\overline{g}] + \frac{\alpha_k^2 L_2}{2}\mathbb{E}_{z_k,\epsilon_{j}}[\|\overline{g}\|_2^2] \\
            & \leq 2\alpha_k\|\nabla R(\theta_k)\|_2^2 \|\nabla^{2}R(\theta_k)\|_2\|\mathbb{E}_{z_k,\epsilon_{j}}[\overline{g}]\|_2 + \frac{\alpha_k^2 L_2}{2}\mathbb{E}_{z_k,\epsilon_{j}}[\|\overline{g}\|_2^2] \\
            & \leq 2\alpha_k L\|\nabla R(\theta_k)\|_2^2 + \frac{\alpha_k^2 L_2}{2}(\frac{m_1}{M} + \frac{m_2+M}{M}\|\nabla R(\theta_k)\|_2^2 ).
\end{aligned}&&                                 
\end{flalign}
Taking total expectation of both sides of Eq.~\ref{eq:substact_F} yields
\begin{equation}\label{eq:bound}
    \mathbb{E}[F(\theta_{k+1})] -  \mathbb{E}[F(\theta_k)] \leq 2\alpha_k L\mathbb{E}[\|\nabla R(\theta_k)\|_2^2] + \frac{\alpha_k^2 L_2}{2}(\frac{m_1}{M} + \frac{m_2+M}{M}\mathbb{E}[\|\nabla R(\theta_k)\|_2^2] ).
\end{equation}

Eq.~\ref{eq:infty} implies that $2\alpha_k L\mathbb{E}[\|\nabla R(\theta_k)\|_2^2]$ is the term of a convergent sum. Besides, $\frac{\alpha_k^2 L_2}{2}(\frac{m_1}{M} + \frac{m_2+M}{M}\mathbb{E}[\|\nabla R(\theta_k)\|_2^2] )$ is also the term of a convergent sum, because $\sum_{k=1}^{\infty}\alpha_k^2$ converges. Hence, the bound (Eq.~\ref{eq:bound}) is also the term of a convergent sum.
Now, let us define
\begin{align}
    & A_{K}^{+} = \sum_{k=0}^{K-1}\max(0, \mathbb{E}[F(\theta_{k+1})] -  \mathbb{E}[F(\theta_k)]), \\
    \mathrm{and}\quad & A_{K}^{-} = \sum_{k=0}^{K-1}\max(0, \mathbb{E}[F(\theta_k)]-\mathbb{E}[F(\theta_{k+1})]).
\end{align}
Because the bound of $\mathbb{E}[F(\theta_{k+1})] -  \mathbb{E}[F(\theta_k)]$ is positive and is the term a of convergent sum, and the sequence $A_{K}^{+}$ is upper bounded by the sum of the bound of $\mathbb{E}[F(\theta_{k+1})] -  \mathbb{E}[F(\theta_k)]$, $A_{K}^{+}$ converges. Similarly, $A_{K}^{-}$ also converges. Since for any $K\in\mathbb{N}$, $F(\theta_K) = F(\theta_0) + A_{K}^{+} - A_{K}^{-}$, we can obtain that $F(\theta_k)$ converges. By Lemma~\ref{lemma:liminf} and the fact that $F(\theta_k)$ converges, we have
\begin{equation}
    \lim_{k\to\infty}\mathbb{E}[\|R(\theta_k)\|_2^2] = 0.
\end{equation}
\end{proof}

\subsection{More Experimental Details} \label{appendix:setup}
In CIFAR-100, each class contains 500 images for training and 100 images for test, with each image of size 32×32. In \emph{mini}ImageNet, each class contains 500 training images and 100 test images of size 84×84. CUB-200-2011 contains 5994 training images and 5794 test images in total with varying number of images for each class, and we resize and crop each image to be of size $224\times224$.

Since there lacks a unified standard in storing/saving exemplars for incremental few-shot learning, we choose the setting that we consider most reasonable and practical. In real-world applications, normally there exists a large number of base classes with sufficient training data (e.g., the base dataset is ImageNet-1K~[5]), whereas the number of unseen novel classes that lack training data is relatively small. Therefore, for computational efficiency and efficient use of storage, it is desirable NOT saving any exemplars for base classes but store some exemplars for new classes. In our experiments, we do not store any exemplar for base classes, but save 5 exemplars for each new class. This will hardly cost any storage space or slow down computation considerably due to the small number of new classes.

\emph{To ensure a fair comparison}, for ICaRL~[4] and Rebalance~[2], we store 2 exemplars per class (for both base classes and new classes). As a result, in each session, 
they store more examplars than our method. For our re-implementation of FSLL~[3], we store the same number of exemplars for each new class as in our method. For other approaches, since the code is not available or the method is too complex to re-implement, we directly use the results reported in their paper, which are substantially lower than the Baseline. 

\input{appendix/table_fig}

\subsection{Additional Experiment Results} \label{appendix:addition_exp}
In this section, we conduct experiments to verify the correctness of our re-implementation of state-of-the-art methods including ICaRL\textsuperscript{$\ast$}~[4], Rebalance\textsuperscript{$\ast$}~[2] and FSLL\textsuperscript{$\ast$}~[3]. Note that we re-implement ICaRL and Rebalance because they (and the released codes) are designed for incremental learning, not for incremental few-shot learning. We re-implement FSLL because the code is not provided. In addition, we present empirical evidence on the difference in the norm of the class prototypes between new classes and base classes, which motivates the design of prototype normalization. 

\textbf{Correctness of our implementation.}
To verify the correctness of our implementation of ICaRL\textsuperscript{$\ast$}~[4] and Rebalance\textsuperscript{$\ast$}~[2], we conduct experiments on CIFAR-100 for \emph{incremental learning}. We adopt 32-layer ResNet as backbone and store 20 exemplars per class as in Rebalance~[2]. The comparative results are presented in Fig~\ref{fig:cifar}. It can be seen that
our re-implementation results of ICaRL and Rebalance are very close to those reported in [2]. 

\input{appendix/table_fsll}

To verify the correctness of our implementation of FSLL [3], we compare the results of our implementation and those reported in~[3] in Table~\ref{table:fsll}. It can be seen that our implementation achieves similar and slightly higher results than those reported in the original paper [3]. Here, the experiments are conducted following the settings in [3] without saving any exemplars for new classes. 

\textbf{Norm of class prototype.} In our experiments, we observe that after training on base classes with balanced data, the norms of the class prototypes of base classes tend to be similar. However, after fine-tuning with very few data on unseen new classes, the norms of the new class prototypes are noticeably smaller than those of the base classes. In Table~\ref{table:norm}, we show the average norms of the prototypes of base classes and new classes after incremental few-shot learning on CIFAR-100, where we randomly select 60 classes as base classes and the remaining 40 classes as new classes.

\input{appendix/table_case1}
\input{appendix/table_case2}

\textbf{Results on CIFAR-100 with different class splits.}
To analyze how difference in patterns of the base and new classes influence our proposed method F2M, we split the classes according to the superclasses and provide results on two different class splits. All the 100 classes of CIFAR-100 are grouped into 20 superclasses, and each superclass contains 5 classes. For {the first class split}, the base classes consist of aquatic mammals, fish, insects, reptiles, small mammals, and large carnivores (30 classes in total). The few-shot novel classes consist of household furniture, vehicles2, flowers, and food containers (20 classes in total). For {the second class split}, the base classes consist of aquatic mammals, fish, insects, reptiles, household furniture, and small mammals (30 classes in total). The few-shot novel classes consist of people, vehicles2, flowers, and food containers (20 classes in total). The experimental results with the two different class splits are presented in Table~\ref{tab:case1_split} and Table~\ref{tab:case2_split} respectively. The results show that even with a large difference between the base classes and novel classes, our F2M still consistently outperforms other methods, indicating its robustness and effectiveness.

\textbf{Error bars of the main results.}
The experimental results reported in  Section~\ref{sec:exp} are the average of 10 runs. For each run, we randomly selected 5 samples for each class (for 5-shot tasks). Here, in Table~\ref{table:error_cifar}, Table~\ref{table:error_mini} and Table~\ref{table:error_cub}, we report the means and 95\% confidence intervals of our method F2M, the Baseline, and the methods that we re-implemented. The confidence intervals indicate that our method F2M achieves steady improvement over state-of-the-art methods.

\input{appendix/table_cifar_error}
\input{appendix/table_mini_error}
\input{appendix/table_cub_error}

\textbf{Results with the same class splits as in TOPIC~[6].}
The experimental results of our F2M and some other methods (our re-implementations) presented in Table~\ref{tab:cifar}, Table~\ref{tab_miniImageNet}, and Table~\ref{tab_cub} are on random class splits with random seed 1997. Here, we conduct experiments using the same class split as in TOPIC~[6]. The experimental results on CIFAR-100, \emph{mini}ImageNet, and CUB-200-2011 are presented in Table~\ref{tab:topic_cifar}, Table~\ref{tab:topic_mini}, and Table~\ref{table:topic_cub} respectively. The results show that 
the Baseline and our F2M still consistently outperform other methods. 
Note that on CUB-200-2011, joint-training outperforms the Baseline and our F2M. The reasons may include: 1) The data imbalance issue is not very significant since the average number of images per class of this dataset is relatively small (about 30); and 2) During the base training stage, we use a smaller learning rate (e.g., 0.001) for the embedding network (pretrained on ImageNet) and a higher learning rate (e.g., 0.01) for the classifier.

\input{appendix/table_cifar}
\input{appendix/table_miniImageNet}
\input{appendix/table_topic_cub}

\section*{References}
[1] Prentice-Hall, Englewood Cliffs. Mathematical Statistics. NJ, 1962.

[2] Saihui Hou, Xinyu Pan, Chen Change Loy, Zilei Wang, and Dahua Lin. Learning a unified classifier incrementally via rebalancing. In \emph{Proceedings of the IEEE/CVF Conference on Computer Vision and Pattern Recognition}, pages 831–839, 2019.

[3] Pratik Mazumder, Pravendra Singh, and Piyush Rai. Few-shot lifelong learning. In \emph{Proceedingsof the AAAI Conference on Artificial Intelligence}, volume 35, pages 2337–2345, 2021.

[4] Sylvestre-Alvise Rebuffi, Alexander Kolesnikov, Georg Sperl, and Christoph H Lampert. icarl: Incremental classifier and representation learning. \emph{In Proceedings of the IEEE conference on Computer Vision and Pattern Recognition}, pages 2001–2010, 2017.

[5] Jia Deng, Wei Dong, Richard Socher, Li-Jia Li, Kai Li, and Li Fei-Fei. Imagenet:  A large-scale hierarchical image database. In \emph{2009 IEEE conference on computer vision and pattern recognition}, pages 248–255. Ieee, 2009.

[6] Xiaoyu Tao, Xiaopeng Hong, Xinyuan Chang, Songlin Dong, Xing Wei, and Yihong Gong. Few-shot class-incremental learning. In \emph{Proceedings of the IEEE/CVF Conference on Computer Vision and Pattern Recognition}, pages 12183–12192, 2020.

[7] Saihui Hou, Xinyu Pan, Chen Change Loy, Zilei Wang, and Dahua Lin.  Learning a unified classifier incrementally via rebalancing. In \emph{Proceedings of the IEEE/CVF Conference on Computer Vision and Pattern Recognition}, pages 831–839, 2019.

[8] Bingyi Kang, Saining Xie, Marcus Rohrbach, Zhicheng Yan, Albert Gordo, Jiashi Feng, and Yannis Kalantidis. Decoupling representation and classifier for long-tailed recognition.  In \emph{International Conference on Learning Representations}, 2019.

[9] Kuilin Chen and Chi-Guhn Lee. Incremental few-shot learning via vector quantization in deep embedded space. In \emph{International Conference on Learning Representations}, 2021.

%% file: appendix/table_fig.tex
\begin{minipage}{.47\textwidth}
    \centering
    \includegraphics[width = \linewidth]{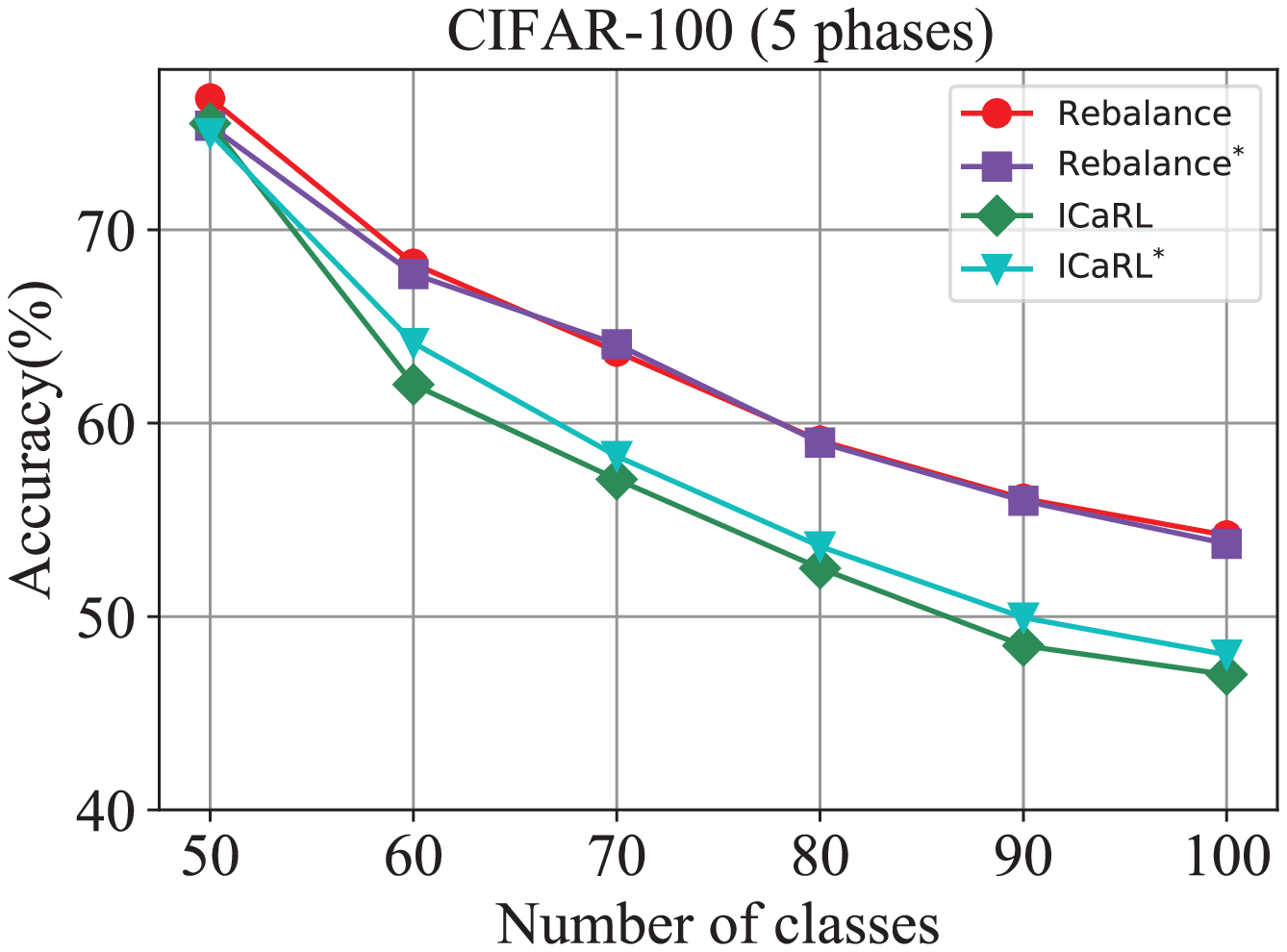}
    \captionof{figure}
    {Our re-implementation results of Rebalance and ICaRL are very close to those reported in~[2]. \textsuperscript{$\ast$}~indicates our re-implementation.}
    \label{fig:cifar}
\end{minipage}
\hfill
\begin{minipage}{.47\textwidth}
\centering
  \captionof{table}{
  The average norm of the class prototypes of new classes is significantly smaller than that of old classes. The experiment is conducted on CIFAR-100 with 60 base classes and 40 new classes. 
  }\label{table:norm}
  \resizebox{1.0\textwidth}{!}{
    \begin{tabular}{lcc}
  \toprule
                & Mean & Standard Deviation\\
  \midrule
  Base classes  & 7.97 & 0.63 \\
  New classes   & 7.48 & 0.71 \\
  \bottomrule
\end{tabular}}
\end{minipage}

%% file: appendix/table_fsll.tex
\begin{table*}[t]
\renewcommand\arraystretch{1.0}
\caption{Our re-implementation results of FSLL are very close to those reported in~[3] on CIFAR-100 for 5-way 5-shot incremental learning. \textsuperscript{$\ast$} indicates our re-implementation. The results are obtained without saving any exemplars.} 
\label{table:fsll}
\centering
\resizebox{0.82\textwidth}{!}{
\begin{tabular}{lccccccccc}
  \toprule
  \multirow[m]{2}{*}{\textbf{Method}} & \multicolumn{9}{c}{\textbf{sessions}} \\
  \cmidrule[0.05em]{2-10}
  & 1 & 2 & 3 & 4 & 5 & 6 & 7 & 8 & 9 \\
  \midrule
  FSLL~[3]\textsuperscript{$\ast$}        & 65.18 & 56.37 & 52.59 & 48.39 & 47.46 & 43.44 & 41.37 & 40.17 & 38.56\\
  FSLL~[3]                                & 64.10 & 55.85 & 51.71 & 48.59 & 45.34 & 43.25 & 41.52 & 39.81 & 38.16\\
  \bottomrule
\end{tabular}
}
\end{table*}

%% file: appendix/table_case1.tex
\begin{table*}[h]
\renewcommand\arraystretch{1.2}
\caption{Classification accuracy for 5-way 5-shot incremental learning with the first class split on CIFAR-100.} 
\label{tab:case1_split}
\centering
\resizebox{0.94\textwidth}{!}{
\begin{tabular}{lcccccc}
  \toprule
  \multirow[m]{2}{*}{\textbf{Method}} & \multicolumn{5}{c}{\textbf{sessions}} & \multirow[m]{2}{*}{\shortstack{\textbf{The gap}\\ \textbf{with cRT}}}\\
  \cmidrule[0.05em]{2-6}
  & 1 (Animals) & 2 (Vehicles2) & 3 (Flowers) & 4 (Food Containers) & 5 (Household Furniture)\\
  \midrule
  Baseline      & 63.07 & 56.32 & 51.40 & 46.85 & 43.55 & -19.52\\
  ICaRL         & \textbf{63.30} & 55.10 & 49.12 & 44.46 & 40.95 & -22.35\\
  Rebalance     & 63.03 & 52.06 & 45.87 & 39.35 & 35.24 & -27.29\\
  FSLL          & 63.07 & 50.72 & 44.53 & 40.73 & 38.00 & -25.07\\
  \textbf{F2M}  & 62.53 & \textbf{56.63} & \textbf{51.87} & \textbf{47.54} & \textbf{44.10} & \textbf{-18.43} \\
  \bottomrule
\end{tabular}
}
\end{table*}

%% file: appendix/table_case2.tex
\begin{table*}[h]
\renewcommand\arraystretch{1.2}
\caption{Classification accuracy for 5-way 5-shot incremental learning with the second class split on CIFAR-100.} 
\label{tab:case2_split}
\centering
\resizebox{0.94\textwidth}{!}{
\begin{tabular}{lcccccc}
  \toprule
  \multirow[m]{2}{*}{\textbf{Method}} & \multicolumn{5}{c}{\textbf{sessions}} & \multirow[m]{2}{*}{\shortstack{\textbf{The gap}\\ \textbf{with cRT}}}\\
  \cmidrule[0.05em]{2-6}
  & 1 (Animals+Furniture) & 2 (People) & 3 (Vehicles2) & 4 (Flowers) & 5 (Food Containers) \\
  \midrule
  Baseline      & 63.07 & 54.30 & 50.16 & 46.19 & 43.16 & -19.91 \\
  ICaRL         & 62.57 & 51.67 & 47.51 & 42.98 & 39.63 & -22.94 \\
  Rebalance     & \textbf{63.50} & 49.62	& 44.67 & 39.68 & 35.64 & -27.86 \\
  FSLL          & 63.07 & 49.45 & 46.30 & 41.94 & 39.33 & -23.74 \\
  \textbf{F2M}  & 62.87 & \textbf{54.82} & \textbf{50.88} & \textbf{46.88} & \textbf{43.83} & \textbf{-19.04} \\
  \bottomrule
\end{tabular}
}
\end{table*}

%% file: appendix/table_cifar_error.tex
\begin{table*}[ht]
\renewcommand\arraystretch{1.0}
\caption{Classification accuracy on CIFAR-100 for 5-way 5-shot incremental learning with 95\% confidence intervals. \textsuperscript{$\ast$} indicates our re-implementation.} 
\label{table:error_cifar}
\centering
\resizebox{0.88\textwidth}{!}{
\begin{tabular}{lccccccccc}
  \toprule
  \multirow[m]{2}{*}{\textbf{Method}} & \multicolumn{9}{c}{\textbf{sessions}} \\
  \cmidrule[0.05em]{2-10}
  & 1 & 2 & 3 & 4 & 5 & 6 & 7 & 8 & 9 \\
  \multirow{1}{*}[0.45em]{Baseline}                                       & \multirow{1}{*}[0.5em]{65.18} & \shortstack[r]{61.67 \\ $\pm$ 0.18} & \shortstack[r]{58.61 \\ $\pm$  0.25} & \shortstack[r]{55.11 \\ $\pm$  0.19} & \shortstack[r]{51.86 \\ $\pm$  0.22} & \shortstack[r]{49.43\\ $\pm$  0.28} & \shortstack[r]{47.60 \\ $\pm$  0.25} & \shortstack[r]{45.64 \\ $\pm$  0.29} & \shortstack[r]{43.83\\ $\pm$  0.22} \\
  \rule{0pt}{5ex}%
  \multirow{1}{*}[0.45em]{iCaRL~[4]\textsuperscript{$\ast$} }     & \multirow{1}{*}[0.5em]{66.52} & \shortstack[r]{57.26 \\ $\pm$  0.17} & \shortstack[r]{54.27 \\ $\pm$  0.25} & \shortstack[r]{50.62 \\ $\pm$  0.29} & \shortstack[r]{47.33 \\ $\pm$  0.27} & \shortstack[r]{44.99\\ $\pm$  0.26} & \shortstack[r]{43.14 \\ $\pm$  0.23} & \shortstack[r]{41.16 \\ $\pm$  0.30} & \shortstack[r]{39.49 \\ $\pm$  0.30} \\
  \rule{0pt}{5ex}%
  \multirow{1}{*}[0.45em]{Rebalance~[2]\textsuperscript{$\ast$} }   & \multirow{1}{*}[0.5em]{\textbf{66.66}} & \shortstack[r]{61.42 \\ $\pm$  0.25} & \shortstack[r]{57.29 \\ $\pm$  0.17} & \shortstack[r]{53.02 \\ $\pm$  0.20} & \shortstack[r]{48.85 \\ $\pm$  0.21} & \shortstack[r]{45.68\\ $\pm$  0.30} & \shortstack[r]{43.06 \\ $\pm$  0.27} & \shortstack[r]{40.56 \\ $\pm$  0.38} & \shortstack[r]{38.35 \\ $\pm$  0.48} \\
  \rule{0pt}{5ex}%
  \multirow{1}{*}[0.45em]{FSLL~[3]\textsuperscript{$\ast$} }       & \multirow{1}{*}[0.5em]{65.18} & \shortstack[r]{56.24 \\ $\pm$  0.35} & \shortstack[r]{54.55 \\ $\pm$  0.28} & \shortstack[r]{51.61 \\ $\pm$  0.36} & \shortstack[r]{49.11 \\ $\pm$  0.40} & \shortstack[r]{47.27\\ $\pm$  0.29} & \shortstack[r]{45.35 \\ $\pm$  0.32} & \shortstack[r]{43.95 \\ $\pm$  0.28} & \shortstack[r]{42.22 \\ $\pm$  0.49} \\
  \rule{0pt}{5ex}%
  \multirow{1}{*}[0.45em]{\textbf{F2M}}                                   & \multirow{1}{*}[0.5em]{64.71} & \shortstack[r]{\textbf{62.05} \\ $\bm\pm$ \textbf{0.19}} & \shortstack[r]{\textbf{59.01} \\ $\bm\pm$ \textbf{0.22}} & \shortstack[r]{\textbf{55.58} \\ $\bm\pm$ \textbf{0.21}} & \shortstack[r]{\textbf{52.55} \\ $\bm\pm$ \textbf{0.25}} & \shortstack[r]{\textbf{49.96} \\ $\bm\pm $ \textbf{0.21}} & \shortstack[r]{\textbf{48.08} \\ $\bm\pm$ \textbf{0.24}} & \shortstack[r]{\textbf{46.28} \\ $\bm\pm$ \textbf{0.24}} & \shortstack[r]{\textbf{44.67} \\ $\bm\pm$ \textbf{0.19}} \\
  \bottomrule
\end{tabular}
}
\end{table*}

%% file: appendix/table_mini_error.tex
\begin{table*}[ht]
\renewcommand\arraystretch{1.0}
\caption{Classification accuracy on miniImageNet for 5-way 5-shot incremental learning with 95\% confidence intervals. \textsuperscript{$\ast$} indicates our re-implementation.} 
\label{table:error_mini}
\centering
\resizebox{0.88\textwidth}{!}{
\begin{tabular}{lccccccccc}
  \toprule
  \multirow[m]{2}{*}{\textbf{Method}} & \multicolumn{9}{c}{\textbf{sessions}} \\
  \cmidrule[0.05em]{2-10}
  & 1 & 2 & 3 & 4 & 5 & 6 & 7 & 8 & 9 \\
  \multirow{1}{*}[0.45em]{Baseline}                                       & \multirow{1}{*}[0.5em]{67.30} & \shortstack[r]{63.18 \\ $\pm$ 0.00} & \shortstack[r]{59.62 \\ $\pm$  0.12} & \shortstack[r]{56.33 \\ $\pm$  0.18} & \shortstack[r]{53.28 \\ $\pm$  0.27} & \shortstack[r]{50.50\\ $\pm$  0.28} & \shortstack[r]{47.96 \\ $\pm$  0.30} & \shortstack[r]{45.85 \\ $\pm$  0.32} & \shortstack[r]{43.88\\ $\pm$  0.27} \\
  \rule{0pt}{5ex}%
  \multirow{1}{*}[0.45em]{iCaRL~[4]\textsuperscript{$\ast$} }     & \multirow{1}{*}[0.5em]{67.35} & \shortstack[r]{59.91 \\ $\pm$  0.15} & \shortstack[r]{55.64 \\ $\pm$  0.20} & \shortstack[r]{52.60 \\ $\pm$  0.30} & \shortstack[r]{49.43 \\ $\pm$  0.32} & \shortstack[r]{46.73\\ $\pm$  0.28} & \shortstack[r]{44.13 \\ $\pm$  0.33} & \shortstack[r]{42.17 \\ $\pm$  0.33} & \shortstack[r]{40.29 \\ $\pm$  0.31} \\
  \rule{0pt}{5ex}%
  \multirow{1}{*}[0.45em]{Rebalance~[2]\textsuperscript{$\ast$} }   & \multirow{1}{*}[0.5em]{\textbf{67.91}} & \shortstack[r]{63.11 \\ $\pm$  0.19} & \shortstack[r]{58.75 \\ $\pm$  0.29} & \shortstack[r]{54.83 \\ $\pm$  0.37} & \shortstack[r]{50.68 \\ $\pm$  0.38} & \shortstack[r]{47.11\\ $\pm$  0.36} & \shortstack[r]{43.88 \\ $\pm$  0.33} & \shortstack[r]{41.19 \\ $\pm$  0.38} & \shortstack[r]{38.72 \\ $\pm$  0.39} \\
  \rule{0pt}{5ex}%
  \multirow{1}{*}[0.45em]{FSLL~[3]\textsuperscript{$\ast$} }       & \multirow{1}{*}[0.5em]{67.30} & \shortstack[r]{59.81 \\ $\pm$  0.42} & \shortstack[r]{57.26 \\ $\pm$  0.55} & \shortstack[r]{54.57 \\ $\pm$  0.58} & \shortstack[r]{52.05 \\ $\pm$  0.49} & \shortstack[r]{49.42\\ $\pm$  0.37} & \shortstack[r]{46.95 \\ $\pm$  0.36} & \shortstack[r]{44.94 \\ $\pm$  0.20} & \shortstack[r]{42.87 \\ $\pm$  0.25} \\
  \rule{0pt}{5ex}%
  \multirow{1}{*}[0.45em]{\textbf{F2M}}                                   & \multirow{1}{*}[0.5em]{67.28} & \shortstack[r]{\textbf{63.80} \\ $\bm\pm$ \textbf{0.10}} & \shortstack[r]{\textbf{60.38} \\ $\bm\pm$ \textbf{0.19}} & \shortstack[r]{\textbf{57.06} \\ $\bm\pm$ \textbf{0.29}} & \shortstack[r]{\textbf{54.08} \\ $\bm\pm$ \textbf{0.28}} & \shortstack[r]{\textbf{51.39} \\ $\bm\pm$ \textbf{0.32}} & \shortstack[r]{\textbf{48.82} \\ $\bm\pm $ \textbf{0.32}} & \shortstack[r]{\textbf{46.58} \\ $\bm\pm$ \textbf{0.33}} & \shortstack[r]{\textbf{44.65} \\ $\bm\pm$ \textbf{0.29}} \\
  \bottomrule
\end{tabular}
}
\end{table*}

%% file: appendix/table_cub_error.tex
\begin{table*}[ht]
\renewcommand\arraystretch{1.0}
\caption{Classification accuracy on CUB-200-2011 for 10-way 5-shot incremental learning with 95\% confidence intervals. \textsuperscript{$\ast$} indicates our re-implementation.} 
\label{table:error_cub}
\centering
\resizebox{\textwidth}{!}{
\begin{tabular}{lccccccccccc}
  \toprule
  \multirow[m]{2}{*}{\textbf{Method}} & \multicolumn{9}{c}{\textbf{sessions}} \\
  \cmidrule[0.05em]{2-12}
  & 1 & 2 & 3 & 4 & 5 & 6 & 7 & 8 & 9 & 10 & 11\\
  \multirow{1}{*}[0.45em]{Baseline}                                       & \multirow{1}{*}[0.5em]{80.87} & \shortstack[r]{77.15 \\ $\pm$ 0.18} & \shortstack[r]{74.46 \\ $\pm$  0.22} & \shortstack[r]{72.26 \\ $\pm$  0.26} & \shortstack[r]{69.47 \\ $\pm$  0.35} & \shortstack[r]{67.18\\ $\pm$  0.27} & \shortstack[r]{65.62 \\ $\pm$  0.38} & \shortstack[r]{63.68 \\ $\pm$  0.25} & \shortstack[r]{61.30\\ $\pm$  0.22} & \shortstack[r]{59.72\\ $\pm$  0.27} & \shortstack[r]{58.12\\ $\pm$  0.27}\\
  \rule{0pt}{5ex}%
  \multirow{1}{*}[0.45em]{iCaRL~[4]\textsuperscript{$\ast$} }     & \multirow{1}{*}[0.5em]{79.58} & \shortstack[r]{67.63 \\ $\pm$  0.25} & \shortstack[r]{64.17\\ $\pm$  0.30} & \shortstack[r]{61.80 \\ $\pm$  0.35} & \shortstack[r]{58.10 \\ $\pm$  0.33} & \shortstack[r]{55.51\\ $\pm$  0.38} & \shortstack[r]{53.34 \\ $\pm$  0.32} & \shortstack[r]{50.89 \\ $\pm$  0.25} & \shortstack[r]{48.62\\ $\pm$  0.29} & \shortstack[r]{47.34\\ $\pm$  0.33} & \shortstack[r]{45.60\\ $\pm$  0.31}\\
  \rule{0pt}{5ex}%
  \multirow{1}{*}[0.45em]{Rebalance~[2]\textsuperscript{$\ast$} }   & \multirow{1}{*}[0.5em]{80.94} & \shortstack[r]{70.32 \\ $\pm$  0.28} & \shortstack[r]{62.96 \\ $\pm$  0.31} & \shortstack[r]{57.19 \\ $\pm$  0.30} & \shortstack[r]{51.06 \\ $\pm$  0.37} & \shortstack[r]{46.70 \\ $\pm$  0.29} & \shortstack[r]{44.03 \\ $\pm$  0.40} & \shortstack[r]{40.15 \\ $\pm$  0.27} & \shortstack[r]{36.75 \\ $\pm$  0.32} & \shortstack[r]{34.88 \\ $\pm$  0.35} & \shortstack[r]{32.09 \\ $\pm$  0.39}\\
  \rule{0pt}{5ex}%
  \multirow{1}{*}[0.45em]{FSLL~[3]\textsuperscript{$\ast$} }       & \multirow{1}{*}[0.5em]{80.83} & \shortstack[r]{77.38 \\ $\pm$  0.30} & \shortstack[r]{72.37 \\ $\pm$  0.25} & \shortstack[r]{71.84 \\ $\pm$  0.45} & \shortstack[r]{67.51 \\ $\pm$  0.42} & \shortstack[r]{65.30 \\ $\pm$  0.50} & \shortstack[r]{63.75 \\ $\pm$  0.39} & \shortstack[r]{61.16 \\ $\pm$  0.28} & \shortstack[r]{59.05 \\ $\pm$  0.37} & \shortstack[r]{58.03 \\ $\pm$  0.35} & \shortstack[r]{55.82 \\ $\pm$  0.33}\\
  \rule{0pt}{5ex}%
  \multirow{1}{*}[0.45em]{\textbf{F2M}}                                   & \multirow{1}{*}[0.5em]{\textbf{81.07}} & \shortstack[r]{\textbf{78.16} \\ $\bm\pm$ \textbf{0.14}} & \shortstack[r]{\textbf{75.57} \\ $\bm\pm$ \textbf{0.24}} & \shortstack[r]{\textbf{72.89} \\ $\bm\pm$ \textbf{0.32}} & \shortstack[r]{\textbf{70.86} \\ $\bm\pm$ \textbf{0.25}} & \shortstack[r]{\textbf{68.17} \\ $\bm\pm $ \textbf{0.39}} & \shortstack[r]{\textbf{67.01} \\ $\bm\pm$ \textbf{0.32}} & \shortstack[r]{\textbf{65.26} \\ $\bm\pm$ \textbf{0.26}} & \shortstack[r]{\textbf{63.36} \\ $\bm\pm$ \textbf{0.24}} & \shortstack[r]{\textbf{61.76} \\ $\bm\pm$ \textbf{0.27}}   & \shortstack[r]{\textbf{60.26} \\ $\bm\pm$ \textbf{0.28}} \\
  \bottomrule
\end{tabular}
}
\end{table*}

%% file: appendix/table_cifar.tex
\begin{table*}[ht]
\renewcommand\arraystretch{1.0}
\caption{Classification accuracy on CIFAR-100 for 5-way 5-shot incremental learning with the same class split as in TOPIC~[6]. \textsuperscript{$\ast$} indicates our re-implementation.} 
\label{tab:topic_cifar}
\centering
\resizebox{0.9\textwidth}{!}{
\begin{tabular}{lcccccccccc}
  \toprule
  \multirow[m]{2}{*}{\textbf{Method}} & \multicolumn{9}{c}{\textbf{sessions}} & \multirow[m]{2}{*}{\shortstack{\textbf{The gap}\\ \textbf{with cRT}}}\\
  \cmidrule[0.05em]{2-10}
  & 1 & 2 & 3 & 4 & 5 & 6 & 7 & 8 & 9 &\\
  \midrule
  cRT~[8]\textsuperscript{$\ast$}            & 72.28	& 69.58 & 65.16 & 61.41 & 58.83 & 55.87 & 53.28 & 51.38 & 49.51
 & -\\
  Joint-training\textsuperscript{$\ast$}            & 72.28 & 68.40 & 63.31 & 59.16 & 55.73 & 52.81 & 49.01 & 46.74 & 44.34 & -5.17\\
  Baseline                                          & 72.28 & 68.01 & 64.18 & 60.56 & 57.44 & 54.69 & 52.98 & 50.80 & 48.70 &	-0.81
 \\
  \midrule
  iCaRL~[4]\textsuperscript{$\ast$}        & 72.05 & 65.35 & 61.55 & 57.83 & 54.61 & 51.74 & 49.71 & 47.49 & 45.03 &	-4.48 \\
  Rebalance~[2]\textsuperscript{$\ast$}      & \textbf{74.45}	& 67.74	& 62.72	& 57.14	& 52.78	& 48.62	& 45.56	& 42.43	& 39.22 &	-10.29 \\
  FSLL~[3]\textsuperscript{$\ast$}          & 72.28 & 63.84 & 59.64 & 55.49 & 53.21 & 51.77 & 50.93 & 48.94 & 46.96 &	-2.55 \\
  iCaRL~[4]     & 64.10 & 53.28 & 41.69 & 34.13 & 27.93 & 25.06 & 20.41 & 15.48 & 13.73 & -35.78\\
  Rebalance~[2]  & 64.10 & 53.05 & 43.96 & 36.97 & 31.61 & 26.73 & 21.23 & 16.78 & 13.54 & -35.97\\
  TOPIC~[6]     & 64.10 & 55.88 & 47.07 & 45.16 & 40.11 & 36.38 & 33.96 & 31.55 & 29.37 & -20.14\\
  FSLL~[3]      & 64.10 & 55.85 & 51.71 & 48.59 & 45.34 & 43.25 & 41.52 & 39.81 & 38.16 & -11.35\\
  FSLL+SS~[3]    & 66.76 & 55.52 & 52.20 & 49.17 & 46.23 & 44.64 & 43.07 & 41.20 & 39.57 & -9.94\\
  \midrule
  \textbf{F2M}          & 71.45 & \textbf{68.10} & \textbf{64.43} & \textbf{60.80} & \textbf{57.76} & \textbf{55.26} & \textbf{53.53} & \textbf{51.57} & \textbf{49.35} & \textbf{-0.16} \\
  \bottomrule
\end{tabular}
}
\end{table*}

%% file: appendix/table_miniImageNet.tex
\begin{table*}[ht]
\renewcommand\arraystretch{1.0}
\caption{Classification accuracy on \emph{mini}ImageNet for 5-way 5-shot incremental learning with the same class split as in TOPIC~[6]. \textsuperscript{$\ast$} indicates our re-implementation.} 
\label{tab:topic_mini}
\centering
\resizebox{0.9\textwidth}{!}{
\begin{tabular}{lcccccccccc}
  \toprule
  \multirow[m]{2}{*}{\textbf{Method}} & \multicolumn{9}{c}{\textbf{sessions}} & \multirow[m]{2}{*}{\shortstack{\textbf{The gap}\\ \textbf{with cRT}}}\\
  \cmidrule[0.05em]{2-10}
  & 1 & 2 & 3 & 4 & 5 & 6 & 7 & 8 & 9 &\\
  \midrule
  cRT~[8]\textsuperscript{$\ast$}            & 72.08 & 68.15 & 63.06 & 61.12 & 56.57 & 54.47 & 51.81 & 49.86 & 48.31
 & -\\
  Joint-training\textsuperscript{$\ast$}            & 72.08	& 67.31 & 62.04 & 58.51 & 54.41 & 51.53 & 48.70 & 45.49 & 43.88
 & -4.43\\
  Baseline                                          & 72.08	& 66.29	& 61.99 & 58.71 & 55.73 & 53.04 & 50.40 & 48.59 & 47.31 & -1.0 \\
  \midrule
  iCaRL~[4]\textsuperscript{$\ast$}        & 71.77 & 61.85 & 58.12 & 54.60 & 51.49 & 48.47 & 45.90 & 44.19 & 42.71 & -5.6 \\
  Rebalance~[2]\textsuperscript{$\ast$}      & \textbf{72.30} & 66.37 & 61.00 & 56.93 & 53.31 & 49.93 & 46.47 & 44.13 & 42.19 & -6.12 \\
  FSLL~[3]\textsuperscript{$\ast$}          & 72.08 & 59.04 & 53.75 & 51.17 & 49.11 & 47.21 & 45.35 & 44.06 & 43.65 & -4.66 \\
  iCaRL~[4]      & 61.31 & 46.32 & 42.94 & 37.63 & 30.49 & 24.00 & 20.89 & 18.80 & 17.21 & -31.10\\
  Rebalance~[2]  & 61.31 & 47.80 & 39.31 & 31.91 & 25.68 & 21.35 & 18.67 & 17.24 & 14.17 & -34.14\\
  TOPIC~[6]      & 61.31 & 50.09 & 45.17 & 41.16 & 37.48 & 35.52 & 32.19 & 29.46 & 24.42 & -23.89\\
  FSLL~[3]       & 66.48 & 61.75 & 58.16 & 54.16 & 51.10 & 48.53 & 46.54 & 44.20 & 42.28 & -6.03\\
  FSLL+SS~[3]    & 68.85 & 63.14 & 59.24 & 55.23 & 52.24 & 49.65 & 47.74 & 45.23 & 43.92 & -4.39\\
  \midrule
  \textbf{F2M}   & 72.05	& \textbf{67.47} & \textbf{63.16} & \textbf{59.70} & \textbf{56.71} & \textbf{53.77} & \textbf{51.11} & \textbf{49.21} & \textbf{47.84} & \textbf{-0.43} \\
  \bottomrule
\end{tabular}
}
\end{table*}

%% file: appendix/table_topic_cub.tex
\begin{table*}[ht]
\renewcommand\arraystretch{1.0}
\caption{Classification accuracy on CUB-200-2011 for 10-way 5-shot incremental learning with the same class split as in TOPIC~[6]. \textsuperscript{$\ast$} indicates our re-implementation.} 
\label{table:topic_cub}
\centering
\resizebox{\textwidth}{!}{
\begin{tabular}{lcccccccccccc}
  \toprule
  \multirow[m]{2}{*}{\textbf{Method}} & \multicolumn{11}{c}{\textbf{sessions}} & \multirow[m]{2}{*}{\shortstack{\textbf{The gap}\\ \textbf{with cRT}}}\\
  \cmidrule[0.05em]{2-12}
  & 1 & 2 & 3 & 4 & 5 & 6 & 7 & 8 & 9 & 10 & 11\\
  \midrule
  cRT~[8]\textsuperscript{$\ast$}      & 77.16 & 74.41 & 71.31 & 68.08 & 65.57 & 63.08 & 62.44 & 61.29 & 60.12 & 59.85 & 59.30 & -\\
  Joint-training\textsuperscript{$\ast$}      & 77.16 & 74.39 & 69.83 & 67.17 & 64.72 & 62.25 & 59.77 & 59.05 & 57.99 & 57.81 & 56.82 & -2.48 \\
  Baseline                                    & 77.16 & \textbf{74.00} & 70.21 & 66.07 & 63.90 & 61.35 & 60.01 & 58.66 & 56.33 & 56.12	& 55.07	& -4.23 \\
  \midrule
  iCaRL~[4]\textsuperscript{$\ast$}     & 75.95	& 60.90	& 57.65	& 54.51	& 50.83 & 48.21 & 46.95 & 45.74 & 43.21 
  & 43.01 & 41.27 & -18.03\\
  Rebalance~[2]\textsuperscript{$\ast$} & \textbf{77.44} & 58.10 & 50.15 & 44.80 & 39.12 & 34.44 & 31.73 & 29.75 & 27.56 & 26.93 & 25.30 & -34.00\\
  FSLL~[3]\textsuperscript{$\ast$}      & 77.16 & 71.85 & 66.53 & 59.95 & 58.01 & 57.00 & 56.06 & 54.78 & 52.24 & 52.01 & 51.47 & -7.83\\
  iCaRL~[4]     & 68.68 & 52.65 & 48.61 & 44.16 & 36.62 & 29.52 & 27.83 & 26.26 & 24.01 & 23.89 & 21.16 & -39.92\\
  Rebalance~[2]   & 68.68 & 57.12 & 44.21 & 28.78 & 26.71 & 25.66 & 24.62 & 21.52 & 20.12 & 20.06 & 19.87 & -41.21\\
  TOPIC~[6]     & 68.68 & 62.49 & 54.81 & 49.99 & 45.25 & 41.40 & 38.35 & 35.36 & 32.22 & 28.31 & 26.28 & -34.80\\
  FSLL~[3]       & 72.77 & 69.33 & 65.51 & 62.66 & 61.10 & 58.65 & 57.78 & 57.26 & 55.59 & 55.39 & 54.21 & -6.87\\
  FSLL+SS~[3]    & 75.63 & 71.81 & 68.16 & 64.32 & 62.61 & 60.10 & 58.82 & 58.70 & 56.45 & 56.41 & 55.82 & -5.26\\
  \midrule
  \textbf{F2M}           & 77.13 & 73.92 & \textbf{70.27} & \textbf{66.37} & \textbf{64.34} & \textbf{61.69} & \textbf{60.52} & \textbf{59.38} & \textbf{57.15} & \textbf{56.94} & \textbf{55.89} & \textbf{-3.41}\\
  \bottomrule
\end{tabular}
}
\end{table*}